\documentclass{article}

\usepackage{iclr2027_conference}

\usepackage[utf8]{inputenc} 
\usepackage[T1]{fontenc}    
\usepackage{hyperref}       
\usepackage{url}            
\usepackage{booktabs}       
\usepackage{amsfonts}       
\usepackage{nicefrac}       
\usepackage{microtype}      
\usepackage{xcolor}         
\usepackage{graphicx}
\usepackage{wrapfig} 
\usepackage{etoc}

\makeatletter
\newenvironment{proof}[1][Proof]{%
  \par\noindent\textit{#1. }\ignorespaces
}{%
  \hfill$\square$\par
}
\makeatother
\usepackage{amssymb}
\usepackage{amsmath}

\newtheorem{theorem}{Theorem}[section]
\newtheorem{definition}[theorem]{Definition}
\newtheorem{proposition}[theorem]{Proposition}
\newtheorem{corollary}[theorem]{Corollary}

\newtheorem{lemma}[theorem]{Lemma}

\newcommand{\R}{\mathbb{R}}

\newcommand{\cP}{\mathcal{P}}

\newcommand{\KL}{\operatorname{KL}}

\title{On the Expressive Power of Transformers for Contextual Relations}

\author{%
  Demián Fraiman\\
  Instituto de Cálculo\\
  Universidad de Buenos Aires\\
  \texttt{fraimandemian@gmail.com} \\
}

\iclrfinalcopy
\begin{document}

\maketitle

\begin{abstract}
Transformers have revolutionized machine learning by making attention a central mechanism for modeling interactions within a context. Despite the central role of attention, the theoretical capabilities of Transformers for representing contextual relations remain unclear. In this work, we address this question by developing a mathematical framework based on probability and optimal transport.
We view a text as a distribution of its representations and attention as a probabilistic relation between them. This perspective reveals a connection between attention normalization and optimal transport: standard softmax normalization produces conditional relations, while Sinkhorn normalization produces joint relations with prescribed marginals. Thus, both mechanisms provide structured probabilistic relations from attention scores.
Under mild conditions, we establish universal approximation results for both settings. We show that Transformer architectures with Sinkhorn normalization can approximate arbitrary contextual relations represented as joint probabilities, while standard softmax Transformers can approximate arbitrary contextual relations represented as conditional probabilities. These results provide a mathematical characterization of the expressive power of Transformers for contextual relations and show how the choice of normalization determines the probabilistic structure of the relations represented by attention.

\end{abstract}

\section{Introduction}

Understanding language requires more than processing words in isolation: the role of each word depends on the context provided by the others. Attention mechanisms were introduced precisely to address this issue. Roughly speaking, attention allows a model to determine which parts of an input are most relevant for interpreting a given word or producing a given output. This idea is especially powerful because it enables the representation of rich, context-dependent interactions between elements of a sequence.

Transformer architectures place attention at the center of their design. Their empirical success in natural language processing and other domains involving structured data is largely due to their ability to model contextual relationships in a flexible and non-sequential manner, which constitutes a key advantage over earlier recurrent neural networks. Despite this success, an important theoretical question remains open: what kinds of contextual relationships can attention mechanisms actually represent? In particular, a precise mathematical characterization of their expressive power is still incomplete.

\begin{figure}[h]
    \centering
    \includegraphics[width=0.98\linewidth]{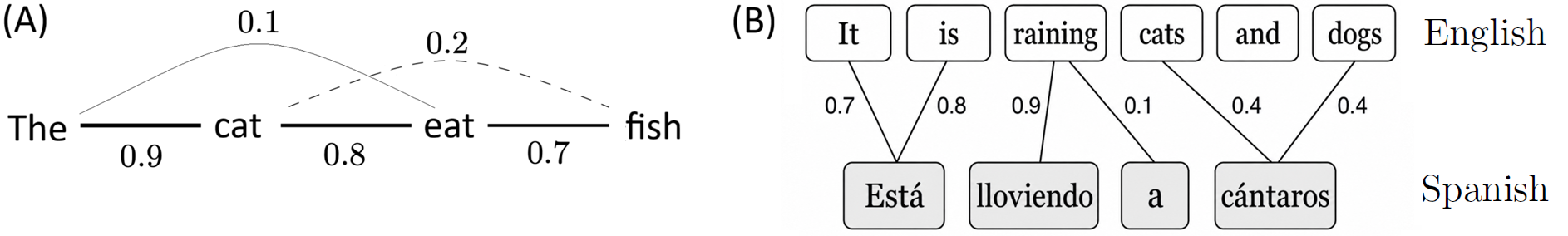}
    \caption{Contextual semantic graphs. (A) Single-text case: words in the same sequence are connected by weighted, possibly asymmetric relations. (B) Two-text case: words from different sequences are linked through cross-text relations. For visual clarity, edges with small weights are omitted.}
\label{fig:contextual_graphs}
\end{figure}

Figure~\ref{fig:contextual_graphs} illustrates the two basic settings that motivate our framework. In both cases, semantic relationships between words are represented as weighted, context-dependent interactions, giving rise to structured relational objects. These two settings naturally correspond to self-attention, where interactions occur within a single sequence, and cross-attention, where interactions link elements across different sequences, as in tasks such as translation and sequence-to-sequence generation. We refer to these structures as \emph{contextual semantic graphs}.

Motivated by these examples, we investigate whether Transformers can represent arbitrary systems of contextual semantic relations. Rather than focusing on individual attention weights, we view these graphs as instances of a more general object: a contextual system that assigns structured, context-dependent relations to each input text or pair of texts. We formalize such systems as mappings that associate to each input a contextual semantic graph, where nodes correspond to words and edges encode context-dependent interactions. This shift in viewpoint allows us to decouple the notion of attention from specific parametrizations and instead study its expressive power at the level of relational structures.

A key challenge in formalizing such systems is that they must operate on sequences of arbitrary length, while remaining invariant to their specific size and representation. This makes standard finite-dimensional formulations insufficient. To handle this problem, we introduce a measure-theoretic framework. In this setting, texts are modeled as probability measures over an embedding space, while contextual relations can be represented as a conditional or joint distributions between texts. Attention mechanisms then naturally appear as operators that map pairs of input measures to conditional distributions. This perspective shifts the focus from pointwise similarity scores between tokens to the approximation of structured probabilistic relations.

Our main result establishes a universal approximation property for contextual semantic relation
systems: attention-based architectures can approximate arbitrarily well any continuous rule that
assigns to each input, or pair of inputs, a contextual relation. The key insight is that Transformer
architectures preserve the full contextual information encoded in the inputs, while the nature of
the output space determines the appropriate final layer. When the goal is to approximate joint
distributions, one must incorporate a Sinkhorn layer, which enforces both marginal constraints and
produces a valid joint distribution; when the objective is to approximate conditional distributions,
standard softmax attention suffices, as it naturally defines probability kernels. The converse
inclusions also hold, so these classes exactly characterize the contextual relations that attention
can represent. Finally, as a corollary, this approximation can also be read at the level of attention matrices:
for any contextual relation, any prescribed accuracy, and any $\delta>0$, there is a single
Transformer block approximating each entry of its attention matrix across all pairs of texts whose
token embeddings are pairwise at distance at least $\delta$, that is, distinguishable enough from
one another.

This paper is structured as follows: in Section 3 we introduce the problem statement. In Section 4, we develop the modeling framework used to address this problem. Section 5 presents the main theorems, while Section 6 gathers the auxiliary results and the sketch of the main theorem proof. We note that these intermediate results may be of independent interest within optimal transport and measure theory.

\section{Related Work}
\paragraph{Expressivity and universality of attention mechanisms.}
The expressive power of attention-based architectures has been studied from multiple perspectives. 
The seminal work of \cite{vaswani2017attention} introduced the Transformer architecture, establishing attention as a central building block for sequence modeling. 
Subsequent works have analyzed its theoretical capabilities: \cite{yun2019transformers} prove universal approximation results for Transformers as sequence-to-sequence models; \cite{perez2021turing} show that Transformers can achieve Turing completeness under suitable conditions. The most closely related works are \cite{furuya2024transformers}, which refines universality results, and \cite{geshkovski2024measure}, which models Transformers as continuous flows capable of interpolating probability distributions. Although we share similar measure-theoretic tools, our objectives differ fundamentally: while \cite{geshkovski2024measure} relies on approximations in the continuous limit and \cite{furuya2024transformers} restricts its universality bounds to the codomain $\mathbb R^d$.

In contrast, our work studies attention at the level of the contextual relations giving key insights in this important property, modeling attention mechanisms as operators that assign joint probability measures to pairs of input distributions. 
As a consequence, classical approximation results in finite-dimensional settings, such as the Stone–Weierstrass theorem, cannot be directly applied, requiring the development of new analytical tools.

\paragraph{Optimal transport in attention and representation learning.}
Optimal transport and entropic regularization have been widely adopted in machine learning as tools for comparing distributions and inducing geometry-aware similarities
\citet{cuturi2013sinkhorn,genevay2016stochastic,peyre2019computational}.
Several works incorporate Sinkhorn iterations into attention mechanisms to improve stability or to encode prior structure
\citet{frogner2015learning,mena2018learning}.
In these approaches, optimal transport is typically used as a computational primitive or loss function acting on learned representations.
\citet{leonard2012schrodinger,nutz2023stability,eckstein2022quantitative} study the analytic properties of the problem and \citet{litman2025scaled} formalizes softmax attention as the exact solution to a one-sided entropic optimal transport problem. \citet{sander2022sinkformers} propose replacing standard softmax attention with Sinkhorn normalization at every layer, enforcing (approximate) double stochasticity throughout the network. 
Similarly, \citet{tay2020sparse} introduce Sparse Sinkhorn Attention, combining Sinkhorn iterations with sparsity constraints to obtain more efficient attention patterns.

Our approach differs in both scope and objective: rather than modifying
all attention layers for architectural design, efficiency, or analytic
stability, we introduce Sinkhorn normalization only at the final
interaction stage, which preserves the expressive structure of classical
Transformers and lets us study the ability of attention-based
architectures to approximate arbitrary Contextual Systems.

\section{Problem Statement}

Our goal is to understand the expressive power of Transformer architectures in modeling contextual relations. To this end, we isolate the minimal architectural component capable of representing such relations, while remaining as close as possible to the standard Transformer design, which we call \emph{Contextual Transformer Block}. This reduction allows us to analyze the core mechanism in a tractable setting, without loss of generality: any more complex architecture that contains this basic structure inherits the same expressive capabilities.

Figure~\ref{fig:contextual-block}.A presents the proposed architecture in the two-text setting. The model takes as input two sequences of tokens (e.g., an input text and an output text), each of which is processed by a standard Transformer encoder (composition of transformer blocks). These encoders produce contextualized embeddings for each token, playing the role of query and key representations. The interaction between both sequences is computed through a similarity score between embeddings, given by a dot product, yielding an affinity matrix exactly like the $QK^\top$ term in standard attention. Finally, a last normalization procedure is applied to the affinity.
\begin{figure}
    \centering
    \includegraphics[width=0.82\linewidth]{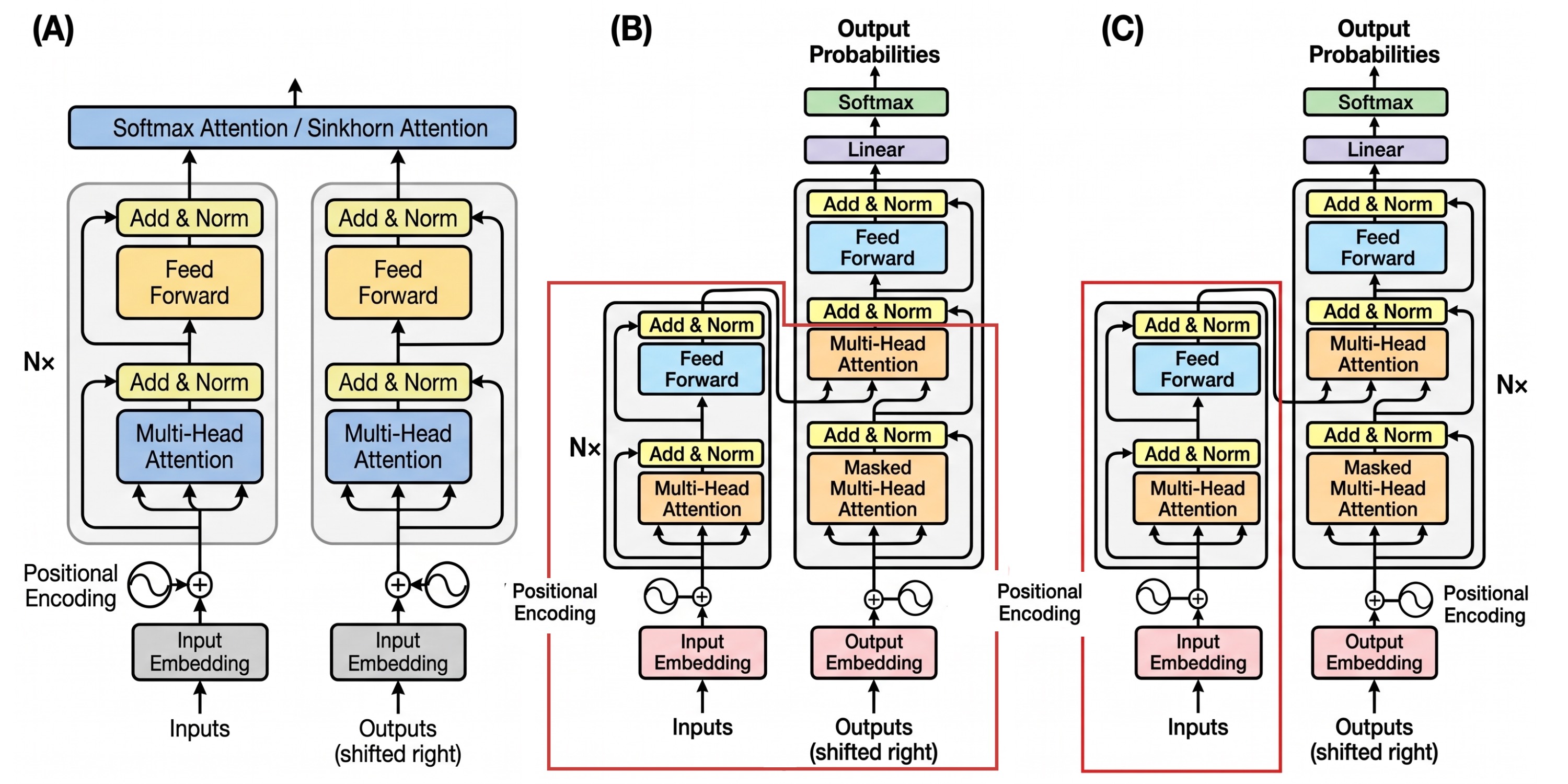}
    \caption{Contextual Transformer Block. (A) Contextual Transformer Block architecture in the two-text setting. (B) Contextual Transformer Block in the original Transformer architecture \cite{vaswani2017attention} two text setting. (C) Contextual Transformer Block in the original Transformer architecture \cite{vaswani2017attention} one text setting.}
    \label{fig:contextual-block}
\end{figure}

Figure~\ref{fig:contextual-block}.B shows how this construction naturally appears as a component within a standard Transformer. In particular, the contextual Transformer can be seen as the minimal block that computes interaction between input and output. The single-text setting (self-attention) is recovered when both inputs coincide, in which case the architecture reduces to a standard Transformer encoder as it is shown in Figure~\ref{fig:contextual-block}.C.

We aim to capture relational structures, and a natural way to model them is through probability distributions. Indeed, for a fixed element, one may interpret its contextual interactions as a distribution of relevance over the other elements, leading to a conditional description of the form “given x, how is attention distributed over y?”. Alternatively, one may model all pairwise relations simultaneously through a joint probability measure over pairs, encoding not only local dependencies but also global consistency constraints.

To make this idea precise, we introduce the notion of a \emph{contextual system}. Informally, a contextual system is a rule that, given two texts, returns how their elements relate to each other. In other words, it assigns to each pair of inputs a structured description of their contextual interactions, either in the form of conditional relations or as a joint distribution over pairs.

In our architecture, a key design choice lies in the final normalization of the affinity matrix, which determines the probabilistic structure of the output. In standard Transformers, a row-wise softmax is applied to the affinity, producing a conditional distribution over tokens. This corresponds to modeling contextual relations as conditional probabilities. In this work, we refer to this case as the \emph{Conditional Transformer Block}, which coincides exactly with the classical transformer block highlighted in \ref{fig:contextual-block}.B.

However, when modeling contextual relations as joint distributions, the output must satisfy both marginal constraints. Standard softmax normalization enforces only one marginal and is therefore insufficient in this setting. To address this, we replace the final softmax with a Sinkhorn normalization \citep{cuturi2013sinkhorn}, which enforces both marginals simultaneously and produces a joint distribution. This modification preserves the structure of standard Transformers in all intermediate layers, altering only the final interaction step. We refer to this architecture as the \emph{Coupling Transformer Block}, as it outputs a coupling between the two input distributions. Empirically, attention matrices in trained models are often observed to be close to doubly stochastic \citep{sander2022sinkformers}, suggesting that enforcing both marginals provides a natural inductive bias. Moreover, Sinkhorn normalization admits an efficient and differentiable implementation via iterative normalization, typically converging in a small number of steps.

The central question addressed in this work is the following:

\emph{Is the contextual transformer block capable of learning (approximating) any contextual system?}

In practice, texts have arbitrary length, and therefore cannot be naturally embedded into a single space $\mathbb R^d$ of fixed dimension. This makes the classical finite-dimensional setting inadequate for our purposes.
To overcome this limitation, we adopt a measure-theoretic perspective, representing texts as probability measures over the embedding space. This allows us to work in a unified framework that is independent of sequence length. Within this setting, we reformulate the notion of contextual relations, as well as the architecture itself.
\section{Measure-Theoretic Framework}
\label{sec:semantic_framework}
Let $X \subset \R^d$ be a compact embedding space.
Given a text consisting of tokens with embeddings  $(x_1,\dots,x_n) \in X$, we associate the
probability measure
\[
\mu = \frac{1}{n}\sum_{i=1}^n \, \delta_{x_i} \in \mathcal P (X),
\]Where $\mathcal P (X)$ is the space of probability measures over $X$. Note that the position of the tokens could be encoded as a positional embedding in the space of $X$. This representation embeds texts of arbitrary length into a common space and allows us to compare them using tools from optimal transport. In particular, we use the Wasserstein distance, which measures the minimal cost of transporting one distribution into another with respect to their embedding distance in $X$.
\begin{definition}[Wasserstein distance]\citep{villani2009optimal}
Let $p\geq1$, the Wasserstein distance of order $p$ is defined as
\[
W_p(\mu,\nu) = \Bigl( \inf_{\pi\in\Pi(\mu,\nu)} \mathbb{E}_{(X,Y)\sim\pi}\bigl[d(X,Y)^p\bigr] \Bigr)^{\!1/p},
\]
\end{definition}
where $\Pi(\mu,\nu)$ denotes the set of couplings (joint measures with marginals $\mu$ and $\nu$). Beyond this intuition, the Wasserstein distance has strong theoretical properties: it metrizes weak convergence together with convergence of moments of order $p$, making it particularly well-suited for comparing probability measures in our setting.

\subsection{Transformer Encoders}

Within this framework, we extend the Transformer encoder architecture to operate directly on probability measures. We build upon the construction of \cite{furuya2024transformers}, which provides a natural generalization of Transformer encoder layers in this setting.
\begin{definition}[Measure-valued multi-head attention]
Given a probability measure $\mu \in \mathcal{P}(X)$ and a point $x \in X \subset \mathbb R^d$, we define the attention map
$
\Gamma_\theta : \mathcal{P}(X) \times X \to \mathbb{R}^{d}
$
by
\begin{equation*}\label{eq:measure-attention}
\Gamma_\theta(\mu,x)
=
x + \sum_{h=1}^H W^h
 {\int_X \frac{\exp\!\big( \frac{1}{\sqrt{k_h}}\langle W_Q^h x, W_K^h y\rangle \big)}
{\int_X \exp\!\big( \frac{1}{\sqrt{k_h}}\langle W_Q^h x, W_K^h z\rangle \big)\, d\mu(z)}\, W_V^h y \, d\mu(y),}
\end{equation*}
\end{definition}
where $\theta = (\theta^1,\dots,\theta^H)$, $\theta^h=(W_Q^h,W_K^h,W_V^h,W^h)$. This definition extends classical multi-head attention, recovering the standard formulation when $\mu$ is an empirical measure supported on finitely many points, since the integral with respect to a discrete measure turns into a sum, as it is shown in Appendix~A. While the attention map $\Gamma_\theta(\mu,x)$ is defined pointwise, in practice the inputs $x$ correspond to the tokens of the text represented by $\mu$. Thus, the architecture operates on the text as a whole.

To define deep architectures in this setting, one needs to formalize the composition of attention layers with pointwise nonlinearities. This can be done by introducing suitable notions of context maps and their composition.  For readability, we defer the full technical definition to Appendix~\ref{app:A}. At a high level, a Transformer encoder block is obtained by composing measure-valued attention operators with pointwise MLPs. This defines in context-function $T:\mathcal P(X)\times X \to \mathbb R^d$.

\subsection{Contextual Relations}

For modeling contextual relations, the attention mechanism serves as our primary source of motivation. Since it  assigns weights to pairs of tokens, measuring how strongly one token
attends to another in a row normalized manner.
Let $(x_1,\dots,x_n)\in X$ be a sequence of tokens embeddings. In a standard Transformer, attention coefficients are defined as
\begin{align*}
A(x_i,x_j) = \text{softmax}_j(W_Qx_i,W_Kx_j)\
 = \frac{\exp(\langle W_Qx_i, W_Kx_j\rangle)}{\sum_{j} \exp(\langle W_Qx_i,W_Kx_j\rangle)}.
\end{align*}
Where $W_Q,W_K \in \mathbb R^{k\times d}$ and the softmax function is computed coordinate-wise. Note that $A(x_i,\cdot)$ can be interpreted as a conditional probability distribution over tokens. This naturally leads to modeling contextual relations as conditional distributions.
\begin{definition}[Conditional measure]
A map $K\colon X \to \mathcal P(Y)$ is called a conditional distribution or Markov kernel, and we denote the set of all conditional distributions as $\mathcal{K}(Y|X)$.
\end{definition}
From a probabilistic perspective, conditional and joint representations are equivalent. Given a kernel $K(x,dy)$ and a source measure $\mu$, one can define a joint measure $\mu \otimes K \in \mathcal P(X \times Y)$ by
\(
\mu \otimes K(dx,dy) := K(dy|x)\mu(dx).
\)
Conversely, any joint measure $\pi$ with marginal $\mu$ admits such a disintegration \citep{villani2009optimal}. This equivalence allows us to view contextual relations either as conditional distributions or as joint probability measures, depending on the representation that is more convenient for the analysis. 

In the sequence-to-sequence setting, we represent two texts as probability measures
$\mu\in\mathcal{P}(X)$ and $\nu\in\mathcal{P}(Y)$, and a contextual relation between them can be
modeled either as a coupling $\pi\in\Pi(\mu,\nu)$ or as a conditional distribution
$K\in\mathcal{K}(Y|X)$. The latter carries an implicit requirement: since $K(x)$ describes how the
tokens of the first text relate to those of the second, it should place no mass outside the second
text itself, since it would mean that it is attending to a word that is not there. i.e.\ $\operatorname{supp}K(x)\subset\operatorname{supp}\nu$ $\mu$-a.e.\ $x$. We denote $\Pi_1(\mu,\nu)$ the set of joint measures $\mu\otimes K$ satisfying this requirement. Note that
$\Pi(\mu,\nu)\subset\Pi_1(\mu,\nu)$, as a coupling further fixes the second
marginal.

Since we aim to model global contextual systems, we introduce mappings that assign to each pair of input distributions a consistent representation of their interactions, which we then seek to approximate. 

\begin{definition}[Coupling system]
A \emph{coupling system} is a mapping
$
F \colon \mathcal{P}(X) \times \mathcal{P}(Y) \to \mathcal{P}(X \times Y)
$
such that
$
F(\mu,\nu) \in \Pi(\mu,\nu)
$ for every $\mu,\nu$, and is continuous with respect to the Wasserstein distance
\end{definition}

\begin{definition}[Conditional system]
A \emph{conditional system} is a mapping
\(
F\colon \mathcal{P}(X)\times\mathcal{P}(Y)\to\mathcal{K}(Y|X)
\)
such that
\(
\mu\otimes F(\nu|\mu) \in \Pi_1(\mu,\nu)
\)  for every $\mu,\nu$, and the induced joint mapping is continuous with respect to Wasserstein distance.
\end{definition}

In both cases, the single-text setting (self-attention) is recovered
when the two input measures coincide, $\mu=\nu$. When a statement applies to either notions, with the corresponding reading in each case, we simply speak of a \emph{contextual system}.

\subsection{Normalization Layer}
We have formulated the encoders in this measure-theoretic setting; we now turn to the final normalization layer. Given two inputs $\mu$ and $\nu$, the two Transformer Encoders of the architecture, defined as in section 4.1, produce contextualized embeddings $Q(\mu,x), K(\nu,y) \in \mathbb R^d$. Those will serve as the queries and keys in the last layer, inducing an affinity function
\(
\alpha(x,y) = \langle Q(\mu,x),K(\nu,y)\rangle_{\mathbb R^d}
\) 

In order to approximate a contextual system, this affinity must be transformed proportionally into a probabilistic object, either a conditional distribution or a joint distribution. This is achieved through a normalization step, implemented via Softmax or Sinkhorn algorithms in the practical scenario. We now define their counterparts in the general measure-theoretic setting. Proofs of the propositions stated in this section are deferred to Appendix~\ref{app:B}.
\begin{definition}[Softmax conditional measure]
Let $\nu \in \mathcal{P}(Y)$, and let $\alpha : X \times Y \to \mathbb{R}$ be measurable such that $e^{\alpha(x,\cdot)} \in L^1(\nu)$ for all $x$. The \emph{softmax conditional kernel} is defined by
\(
\mathrm{Softmax}_{\nu}(\alpha)(dy|x)
=
\frac{e^{\alpha(x,y)}}{\int_Y e^{\alpha(x,y')} d\nu(y')} \, \nu(dy).
\)
\end{definition}

Another equivalent formulation of softmax conditional kernel is based on \cite{litman2025scaled}:
\begin{proposition}[Variational characterization of softmax](\ref{var-char-softmax})
Let $\mu \in \mathcal{P}(X)$ and $\nu \in \mathcal{P}(Y)$. The joint measure induced by the softmax conditional associated with $\alpha$ is the unique minimizer of
\begin{equation*}
\min_{\pi \in \mathcal{P}(X \times Y)\,:\, \pi_X = \mu}
\int_{X \times Y} -\alpha(x,y)\, d\pi(x,y)
+
\mathrm{KL}\big(\pi \,\|\, \mu \otimes \nu\big)
.
\end{equation*}
\end{proposition}
This formulation suggests a natural extension: if one wishes to enforce both marginals, it suffices to restrict the admissible set to $\Pi(\mu,\nu)$. This leads to the entropic optimal transport problem, which correspond to the generalization of the Sinkhorn algorithm.
\begin{equation*}
\min_{\pi \in \Pi(\mu, \nu)}
\int_{X \times Y} -\alpha(x, y) \, d\pi(x, y)
+
 \, \mathrm{KL}\big(\pi \,\|\, \mu \otimes \nu\big)
, \hspace{0.2cm}\mathrm{KL}(P \,\|\, Q) =
\begin{cases}
\displaystyle \int_X \log\!\left(\frac{dP}{dQ}\right) \, dP, & \text{if } P \ll Q, \\
+\infty, & \text{otherwise}.
\end{cases}
\end{equation*}
If the spaces $X,Y$ are polish (complete separable metric spaces) and $\alpha$ is continuous the problem admits a unique minimizer $\pi$ called the \emph{Sinkhorn plan} with density with respect to the product measure (\cite{leonard2012schrodinger}),
$
\frac{d\pi}{d(\mu \otimes \nu)}(x,y)
=
u(x)\, e^{\alpha(x,y)}\, v(y),
$
where $u,v > 0$ are measurable potentials. The mapping $S_{\alpha} : (\mu,\nu) \mapsto \pi$ will be called the \emph{Sinkhorn operator}. When convenient, we will also use the equivalent cost formulation $c = -\alpha$.

This expression highlights a key structural similarity with softmax: both involve exponentiation of a compatibility function. However, while softmax performs a normalization along $y$ for each fixed $x$ (yielding a conditional distribution), Sinkhorn performs a \emph{bi-normalization} that enforces both marginals simultaneously. We now make explicit the connection between both normalization procedures.
\begin{proposition}(\ref{sink-softma})
Let $\pi = S_\alpha(\mu,\nu)$ be the Sinkhorn plan associated with $\alpha$. Then the conditional measure:
\[
\pi(dy|x)
=
\mathrm{Softmax}_{\nu}\!\left(\log \frac{d\pi}{d(\mu \otimes \nu)}(x,\cdot)\right)(dy|x).
\]

\end{proposition}
\subsection{Contextual Transformer Block}

We now present the generalization of the complete architecture into the measure-theoretic setting, which coincides with the practical implementation in the case of discrete uniform measures. The architecture consists of two Transformer-based encoders $Q$ and $K$, defined as in section 4.1, which map input measures into query and key representations of the tokens for the last layer $
Q(\mu, x), K(\nu, y) \in \mathbb{R}^d.
$ From these embeddings, we define the affinity function:
\[
\alpha_{(\mu,\nu)}(x,y) = \langle Q(\mu, x), K(\nu, y) \rangle_{\mathbb{R}^d},
\]
like $QK^\top$ interaction in standard attention. The only distinction lies in the final normalization: Where Softmax Attention yields a conditional distribution, whereas Sinkhorn Attention produces a joint distribution.{\renewcommand{\arraystretch}{0.95} \[
\begin{array}{cc}
\text{\em Conditional Transformer}
& 
\text{\em Coupling Transformer} \\[4pt]

\mathcal T:\mathcal P(X)\times\mathcal P(Y)\to\mathcal K(Y|X)
&
\mathcal T:\mathcal P(X)\times\mathcal P(Y)\to\mathcal P(X\times Y) \\[6pt]

\mathcal T(\nu|\mu)(dy|x)=\mathrm{Softmax}_{\nu}\!\big(\alpha_{(\mu,\nu)}(x,\cdot)\big)(dy|x)
&
\mathcal T(\mu,\nu)=S_{\alpha(\mu,\nu)}(\mu,\nu)

\end{array}
\]
}

In both cases, the inputs are solely the two texts, represented as probability measures $\mu$ and $\nu$. The variables $x$ and $y$ appearing in the definitions correspond to tokens drawn from these measures and do not represent additional inputs. Rather, the affinity function is evaluated over all pairs of tokens from the two texts, as in standard attention mechanisms. 
\section{Main Theorems}
Now we state the main results of the paper. Full proofs are deferred to Appendices~\ref{app:D} and \ref{app:kernel_systems}. Throughout this section, we make the following standing assumption: $X \subset \mathbb{R}^d$ and $Y \subset \mathbb{R}^{d'}$ are compact sets.

\begin{theorem}[Universal Approximation of Coupling Systems] (\ref{teo_principal})
\label{thm:main}
For any $F$ coupling system, $p \geq 1$ and any $\varepsilon > 0$, there exists a Coupling Transformer Block
$\mathcal{T}^*$ such that
\[
\sup_{(\mu,\nu)\in P(X)\times P(Y)}
W_p\bigl(\mathcal{T}^*(\mu,\nu), F(\mu,\nu)\bigr) < \varepsilon.
\]
\end{theorem}

\begin{theorem}[Universal Approximation of Conditional Systems](\ref{thm:F_main})
    For any $F$  conditional system, $p\geq1$ and any $\varepsilon > 0$, there exists a Conditional Transformer Block $\mathcal T^*$ such that 
\[
\sup_{(\mu,\nu) \in \mathcal P(X) \times \mathcal P(Y)} W_p\bigl(\mu\otimes \mathcal{T}^*(\nu|\mu), \mu \otimes F(\nu|\mu)\bigr) < \varepsilon.
\]
\end{theorem}

Moreover, since continuity, support and marginals are preserved under limits
in this metric, contextual systems exactly
characterize the class of relations a Contextual Transformer Block can approximate,
so that a contextual relation outside this class cannot be uniformly
approximated by a transformer in this sense (Propositions~\ref{prop_contra_coupling}
and~\ref{prop:F_closed}).

Our result should be interpreted as showing that a single Contextual Transformer Block can represent all contextual relations simultaneously as it is approximating the whole contextual system. This also immediately implies the following practical corollary. We identify each text with its corresponding empirical measure.

\begin{corollary}[Approximation of Conditional Matrices]
(\ref{cor:F_practical})
For any $F$ conditional system and any $\varepsilon,\delta>0$, there exists a Conditional
Transformer Block $\mathcal{T}^*$, such that for
\emph{every} pair of texts $(t_1,t_2)$ whose token embeddings are separated by at least $\delta$ within each
text, we have
\[
\max_{i,j}\big|F_{ij}(t_2\mid t_1)-\mathcal{T}^*_{ij}(t_2\mid t_1)\big|<\varepsilon ,
\]
where $F_{ij}(t_2\mid t_1)$ and $\mathcal{T}^*_{ij}(t_2\mid t_1)$ denote the entries of the
corresponding attention matrices.
\end{corollary}

Taken together, these results show that Transformer architectures preserve the full contextual information of the system at the level of representations. The distinction between modeling joint or conditional relations is entirely determined by the final normalization layer. In particular, by modifying only this final step, the same architecture can approximate either coupling systems or conditional systems. We can also conclude that the encoder component alone is sufficiently expressive to capture arbitrary contextual relations within a single input text. When combined with the decoder and a cross-attention mechanism, the resulting architecture can model general contextual relations between the input and output.


\section{Proof Sketch and Results Overview}

This section is addressed to the mathematically inclined reader; it may be
skipped without loss of continuity by those interested only in the statements
and their consequences. We outline the strategy common to the proofs of both
main theorems, deferring the full arguments to the appendices C and D. Along the way we
establish several results of independent interest in measure theory and optimal
transport. The strategy itself appears to be new in the expressivity
literature, and it may be useful for proving universality in other settings,
mainly when the codomain is not Euclidean and tools such as the
Stone--Weierstrass theorem cannot be applied to the target directly.

\begin{proof}[Proof sketch]
Both proofs follow the same four steps.

\emph{First}, the Contextual Transformer of Section~4.4 is a particular case of
the Sinkhorn (respectively, softmax) operator, with affinity
$\alpha_{(\mu,\nu)}(x,y)=\langle Q(\mu,x),K(\nu,y)\rangle$, where $Q,K$ are
Transformer encoders. It is therefore necessary that such operators, with
$\alpha_{(\mu,\nu)}\colon X\times Y\to\R$ an arbitrary continuous function of
the input measures, can approximate any contextual system. We establish this more generally, for
$X,Y$ arbitrary compact metric spaces, and only specialize to
$X\subset\R^d$, $Y\subset\R^{d'}$ when realizing the affinity by encoders.

\emph{Second}, we approximate \emph{pointwise}, i.e.\ we approximate $F(\mu,\nu)$
for a \emph{fixed} pair of measures. We begin by replacing the target with a
measure that is absolutely continuous with respect to $\mu\otimes\nu$: in the coupling case this is Lemma~\ref{lema_aprox_abs_cont},
which averages the plan over a partition of $X\times Y$ into small rectangles,
and in the conditional case it is Lemma~\ref{lem:F_ball}, which averages each
conditional law over $\nu$-balls. We then make the density continuous and
bounded below by a constant depending only on $\varepsilon$
(Lemma~\ref{lema_regularizacion} for couplings, Lemma~\ref{lem:F_reg} for
conditionals), relying on the stability of the Sinkhorn operator
(Lemma~\ref{prop_gamma_convergence}; \citealp{nutz2023stability,
eckstein2022quantitative}) and of the softmax operator
(Lemma~\ref{lem:F_softmax}(ii)), respectively. Finally, such a density is
realized \emph{exactly} by the operator:
$\mathcal{S}_{-\log f}(\mu,\nu)=f\cdot(\mu\otimes\nu)$ by
Lemma~\ref{lema_representation_entropic}, and
$\mu\otimes\mathrm{Softmax}_\nu(\log f)=f\cdot(\mu\otimes\nu)$ by
Lemma~\ref{lem:F_softmax}(i). This gives the pointwise result, recorded in
Corollaries~\ref{coro_aprox_puntual} and~\ref{aprox_softmax_puntual}.

\emph{Third}, we pass to uniform approximation by an approximate selection
argument, applied to the ad hoc set-valued map of \emph{densities}
\[
\Gamma_\varepsilon(\mu,\nu)=\Big\{f\in C(X\times Y):\ f>\eta(\varepsilon),\
f\cdot(\mu\otimes\nu)\in\mathcal{C}_{(\mu,\nu)},\
f\cdot(\mu\otimes\nu)\in B_{W_1}\big(F(\mu,\nu),\tfrac{\varepsilon}{2}\big)\Big\},
\]
where $\mathcal{C}_{(\mu,\nu)}$ is $\Pi(\mu,\nu)$ in the coupling case and
$\Pi_1(\mu;\nu)$ in the conditional one (with $F$ replaced by its joint version
$\mu \otimes F$). Its values are non-empty by the pointwise approximation, and
convex because the constraint sets and the Wasserstein balls are convex. The
remaining hypothesis is lower hemicontinuity, proved in
Proposition~\ref{hemi_gamma_copuling} for couplings and in
Proposition~\ref{prop:F_Gamma} for kernels. The theorem of
\citet{deutsch1983continuous} then yields a continuous approximate selection
$s$, and $\mathcal{S}_{-\log s}$ (respectively, $\mathrm{Softmax}(\log s)$) is
our approximator (Theorem~\ref{teo_aprox_uniforme}, respectively
Theorem~\ref{thm:F_D6}).

\emph{Finally}, $\log s$ is a continuous map from $\mathcal{P}(X)\times
\mathcal{P}(Y)$ to $C(X\times Y)$, which is equivalent to a continuous function
on $\mathcal{P}(X)\times\mathcal{P}(Y)\times X\times Y$, a compact space. The
Stone--Weierstrass theorem, applied to this scalar function rather than to the
measure-valued target, approximates it by finite sums of products
$\langle G(\mu,x),H(\nu,y)\rangle$. We then approximate $G$ and $H$ by
Transformer encoders using \citet{furuya2024transformers}. The stability of the
two operators, quantified in \citet{eckstein2022quantitative} for Sinkhorn and in
Lemma~\ref{lem:F_softmax}(ii) for softmax, transfers these errors to the output,
which concludes both proofs.
\end{proof}

\section{Conclusion and Future Work}
We introduced a measure-theoretic framework to formalize contextual relations between texts,
modeling them either as joint distributions (couplings) or conditional distributions. Within this
framework, we established universal approximation theorems in Wasserstein distance, showed that
these two classes exactly characterize the contextual relations attention can represent, and
derived entrywise approximation results for attention matrices.

From a practical perspective, these results provide insight into the empirical success of
Transformers. They show that the architecture has sufficient expressive power to approximate
arbitrary contextual relations between inputs, and that no relation outside these classes is within
its reach. In particular, the model parameters encode the underlying relational structure, while the
choice of normalization determines its representation: as a conditional distribution (via softmax)
or as a joint distribution (via Sinkhorn normalization). 

Moreover, our analysis provides insight into
the roles of the different components of the architecture. The encoder alone is sufficiently
expressive to capture contextual relations within a single input text. When combined with the
decoder, the resulting architecture can model general contextual relations between input and output
texts. This shows that Transformer architectures are well-suited to represent both intra-text and
inter-text dependencies within a unified framework.

Our work focuses on representational expressivity and does not address questions of learning dynamics, statistical efficiency, or generalization. Understanding how such representations can be learned from finite data, and under what conditions they arise in practice, remains an open problem.

The main open direction is to obtain quantitative versions of our results, such as convergence rates and complexity bounds. Also, it would be of interest to study dynamic or autoregressive settings, where contextual relations evolve sequentially and depend on previously generated tokens.

\section*{AI Use Statement}

In this work, we used generative AI tools to assist with translation. Since the authors are not native English speakers,
an LLM was used to rephrase author-drafted sentences into fluent English,
with the authors verifying that the technical content was preserved at
every step.

An LLM was also used to assist in the writing of proofs. The authors specified every mathematical statement to be proved and the
intended proof. The LLM's role was to transcribe these proofs
into \LaTeX{}, at most filling in short clarifications on why a particular
step of the original proof holds. Its role was larger in two places. In
Lemma~\ref{lem:F_softmax}, the authors provided the strategy and the LLM
completed the resulting algebraic computations; in
Corollary~\ref{cor:F_practical}, the LLM helped write the standard
argument bounding the total variation distance by the Wasserstein distance
divided by the separation constant $\delta$. Every AI-assisted step was
independently verified by the authors, and any errors found during this
review were corrected by the authors. The theoretical framework and
conceptual approach of the paper were not developed with AI assistance, the
LLM did not provide critical ingredients for any proof, and all
definitions, theorems, and their statements are the authors' own.

Additionally, an LLM was used to edit drafts of the paper for grammar,
clarity, and readability, to help create and modify the paper's figures, to
identify and summarize relevant related work, to format references, and to
suggest the paper's title. It was not used to suggest the structure of the
paper or to discover the research topic. The remaining required-disclosure tasks are not
applicable to this work.

We have reviewed all AI-assisted work; every proof step was checked line by
line by the authors, and every figure, literature summary, and reference
was verified against the original sources. We take responsibility for the
final content of this work, including all text, mathematical claims, and
figures produced with the aid of generative AI.

\bibliography{bibliografia}
\bibliographystyle{plainnat}

\newpage
\appendix
\onecolumn

\section{Transformer Encoder Full Definition}\label{app:A}
\begin{definition}[Measure-valued multi-head attention]
Given a probability measure $\mu \in \mathcal{P}(X)$ and a point $x \in X \subset \mathbb R^d$, we define the attention map
$
\Gamma_\theta : \mathcal{P}(X) \times X \to \mathbb{R}^{d}
$
by
\begin{equation*}\label{eq:measure-attention}
\Gamma_\theta(\mu,x)
=
x + \sum_{h=1}^H W^h
\int_X \operatorname{softmax}^{\theta^h}_\mu(x)(y)\, W_V^h y \, d\mu(y),
\end{equation*}
where $\theta = (\theta^1,\dots,\theta^H)$, $\theta^h=(W_Q^h,W_K^h,W_V^h,W^h)$, and
\begin{equation*}\label{eq:softmax-measure}
\operatorname{softmax}^{\theta^h}_\mu(x)(y)
=
\frac{\exp\!\big( \frac{1}{\sqrt{k_h}}\langle W_Q^h x, W_K^h y\rangle \big)}
{\int_X \exp\!\big( \frac{1}{\sqrt{k_h}}\langle W_Q^h x, W_K^h z\rangle \big)\, d\mu(z)}.
\end{equation*}
\end{definition}

This definition extends classical multi-head attention, recovering the standard formulation when $\mu$ is an empirical measure supported on finitely many points, since the integral with respect to a a discrete measure turns into a sum.

\begin{align*}
\int_X \operatorname{softmax}_{\frac{1}{n}\sum_{i=1}^n\delta_{x_i}}^{\theta^h}(x)(y)\,W_V^h y\,d(\frac{1}{n}\sum_{i=1}^n\delta_{x_i})(y)
=
\frac{1}{n}\sum_{j=1}^n
\frac{\exp\!\big(\frac{1}{\sqrt{k_h}}\langle W_Q^h x,\,W_K^h x_j\rangle\big)}
{\frac{1}{n}\sum_{i=1}^n \exp\!\big(\frac{1}{\sqrt{k_h}}\langle W_Q^h x,\,W_K^h x_i\rangle\big)}
\,W_V^h x_j.
\end{align*}
Since the $\frac{1}{n}$ factors cancel each other out, we recover the classical case.
\begin{definition}[Context maps and context operators]
Let $X$ and $Y$ be locally compact metric spaces. A map
\[
\Gamma : \mathcal{P}(X) \times X \to Y
\]
is called a \emph{context map} if $\Gamma(\mu,\cdot)$ is continuous on $X$ for every $\mu \in \mathcal{P}(X)$.

The associated \emph{context operator} $\mathcal{C}(\Gamma) \colon \mathcal{P}(X) \to \mathcal{P}(Y)$ is defined by
\[
\mathcal{C}(\Gamma)(\mu) = \Gamma(\mu,\cdot)_\# \mu.
\]
Where $\#$ denotes the pushforward measure.
\end{definition}

\begin{definition}[Composition of context maps]
Let $\Gamma_1 \colon \mathcal{P}(X)\times X \to Y$ and $\Gamma_2 \colon \mathcal{P}(Y)\times Y \to Z$ be context maps. Their composition
\[
\Gamma_2 \diamond \Gamma_1 \colon \mathcal{P}(X)\times X \to Z
\]
is defined by
\[
(\Gamma_2 \diamond \Gamma_1)(\mu,x)
=
\Gamma_2(\mu_1,\Gamma_1(\mu,x)),
\qquad
\mu_1 = \mathcal{C}(\Gamma_1)(\mu).
\]
\end{definition}

Now we can state the formal definition of the Deep Transformer architecture in the measure-theoretic setting.

\begin{definition}[Transformer Encoder]
A Transformer with $L$ layers is defined as the composition
\[
F_{\xi_L} \diamond \Gamma_{\theta_L} \diamond \cdots \diamond F_{\xi_1} \diamond \Gamma_{\theta_1},
\]
where each $\Gamma_{\theta_\ell}$ is a measure-valued attention map of the form \eqref{eq:measure-attention} and each $F_{\xi_\ell} : \mathbb{R}^{d_\ell} \to \mathbb{R}^{d_{\ell+1}}$ is a multilayer perceptron acting pointwise.
\end{definition}

\section{Softmax Properties}\label{app:B}

\begin{proposition}[Variational characterization of softmax]\label{var-char-softmax}
Let $X,Y$ be polish spaces, $\mu \in \mathcal{P}(X)$ and $\nu \in \mathcal{P}(Y)$, and let $\alpha : X \times Y \to \mathbb{R}$ be measurable such that $e^{\alpha(x,\cdot)} \in L^1(\nu)$ for $\mu$-a.e.\ $x$. Then the unique minimizer of
\begin{equation*}
\min_{\pi \in \mathcal{P}(X \times Y)\,:\, \pi_X = \mu}
\int_{X \times Y} -\alpha(x,y)\, d\pi(x,y)
+
\mathrm{KL}\big(\pi \,\|\, \mu \otimes \nu\big)
\end{equation*}
is given by
\[
\pi(dx,dy) = \mu(dx)\,
\frac{e^{\alpha(x,y)}}{\int_Y e^{\alpha(x,y')} \nu(dy')}
\, \nu(dy).
\]
In particular,
\[
\pi(dy|x) = \mathrm{Softmax}_\nu\big(\alpha(x,\cdot)\big)(dy|x).
\]
\end{proposition}

\begin{proof}
We solve the constrained minimization problem
\[
\min_{\pi \in \mathcal{P}(X \times Y)\,:\, \pi_X = \mu}
\int -\alpha(x,y)\, d\pi(x,y)
+
\mathrm{KL}(\pi \| \mu \otimes \nu).
\]

\medskip

\textbf{Step 1: Disintegration of $\pi$.}

Since $\pi_X = \mu$, by the disintegration theorem \cite{villani2009optimal} there exists a Markov kernel $K(x,dy)$ such that
\[
\pi(dx,dy) = \mu(dx)\, K(x,dy).
\]

\medskip

\textbf{Step 2: Rewrite the objective.}

We compute each term.

First,
\[
\int -\alpha(x,y)\, d\pi(x,y)
=
\int_X \int_Y -\alpha(x,y)\, K(x,dy)\, \mu(dx).
\]

Second, using the chain rule for relative entropy,
\[
\mathrm{KL}(\pi \| \mu \otimes \nu)
=
\int_X \mathrm{KL}\big(K(x,\cdot)\,\|\, \nu\big)\, \mu(dx).
\]

Therefore, the objective becomes
\[
\int_X
\left[
\int_Y -\alpha(x,y)\, K(x,dy)
+
\mathrm{KL}\big(K(x,\cdot)\,\|\, \nu\big)
\right]
\mu(dx).
\]

\medskip

\textbf{Step 3: Pointwise minimization.}

The problem decouples in $x$. For each fixed $x$, we must solve
\[
\min_{K(x,\cdot) \in \mathcal{P}(Y)}
\left\{
\int_Y -\alpha(x,y)\, K(x,dy)
+
\mathrm{KL}\big(K(x,\cdot)\,\|\, \nu\big)
\right\}.
\]

\medskip

\textbf{Step 4: Solve the variational problem.}

Let $p(dy) = K(x,dy)$. The problem becomes
\[
\min_{p \in \mathcal{P}(Y)}
\left\{
\int_Y -\alpha(x,y)\, p(dy)
+
\int_Y \log\!\left(\frac{dp}{d\nu}(y)\right)\, p(dy)
\right\}.
\]

Assuming $p \ll \nu$, we write $p = f \nu$ with $f \ge 0$ and $\int f d\nu = 1$. Then the objective is
\[
\int_Y \left[ -\alpha(x,y) f(y) + f(y)\log f(y) \right] \nu(dy).
\]

We minimize this functional under the constraint $\int f d\nu = 1$. Introducing a Lagrange multiplier $\lambda$, we consider
\[
\mathcal{L}(f)
=
\int_Y \left[ -\alpha(x,y) f(y) + f(y)\log f(y) \right] d\nu(y)
+ \lambda\left(\int f d\nu - 1\right).
\]

Taking first variation, we obtain the optimality condition:
\[
-\alpha(x,y) + \log f(y) + 1 + \lambda = 0.
\]

Thus,
\[
\log f(y) = \alpha(x,y) - (1+\lambda),
\quad
f(y) = C\, e^{\alpha(x,y)}.
\]

The normalization condition $\int f d\nu = 1$ gives
\[
C^{-1} = \int_Y e^{\alpha(x,y')} \nu(dy').
\]

Therefore,
\[
K(x,dy)
=
\frac{e^{\alpha(x,y)}}{\int_Y e^{\alpha(x,y')} \nu(dy')}
\, \nu(dy).
\]

\medskip

\textbf{Step 5: Conclusion and uniqueness.}

Plugging back into $\pi(dx,dy) = \mu(dx)K(x,dy)$ yields the desired expression.

Uniqueness follows from the strict convexity of the KL divergence.
\end{proof}
\begin{proposition}\label{sink-softma}
Let $X,Y$ be polish spaces, $\alpha\colon X \times Y \to \R$ a continuous affinity function, $\mu \in \mathcal P(X)$, $\nu \in \mathcal P(Y)$ and $\pi = S_\alpha(\mu,\nu)$ the Sinkhorn plan associated with $\alpha$. Then
\[
\pi(dy|x)
=
\mathrm{Softmax}_{\nu}\!\left(\log \frac{d\pi}{d(\mu \otimes \nu)}(x,\cdot)\right)(dy|x).
\]
\end{proposition} 

\begin{proof}
From the classical characterization of the Sinkhorn solution \cite{leonard2012schrodinger}, the optimal plan admits the density
\[
\frac{d\pi}{d(\mu \otimes \nu)}(x,y)
=
u(x)\, e^{\alpha(x,y)}\, v(y),
\]
for some positive potentials $u,v$.

Taking logarithms,
\[
\log \frac{d\pi}{d(\mu \otimes \nu)}(x,y)
=
\log u(x) + \alpha(x,y) + \log v(y).
\]

Now fix $x$. The conditional distribution of $\pi$ is
\[
\pi(dy|x)
=
\frac{d\pi}{d\mu}(x,dy)
=
\frac{d\pi}{d(\mu \otimes \nu)}(x,y)\, \nu(dy)
\Big/
\int_Y \frac{d\pi}{d(\mu \otimes \nu)}(x,y')\, \nu(dy').
\]

Plugging the factorization,
\[
\pi(dy|x)
=
\frac{u(x)\, e^{\alpha(x,y)}\, v(y)}{\int_Y u(x)\, e^{\alpha(x,y')}\, v(y') \nu(dy')}
\, \nu(dy).
\]

The factor $u(x)$ cancels, yielding
\[
\pi(dy|x)
=
\frac{e^{\alpha(x,y) + \log v(y)}}{\int_Y e^{\alpha(x,y') + \log v(y')} \nu(dy')}
\, \nu(dy).
\]

This is exactly a softmax with respect to the function
\[
y \mapsto \alpha(x,y) + \log v(y)
=
\log \frac{d\pi}{d(\mu \otimes \nu)}(x,y) - \log u(x).
\]

Since adding a constant in $y$ does not change softmax, we obtain
\[
\pi(dy|x)
=
\mathrm{Softmax}_\nu\!\left(\log \frac{d\pi}{d(\mu \otimes \nu)}(x,\cdot)\right)(dy|x).
\]
\end{proof}

\section{Coupling Systems}\label{app:D}
Throughout this section we will use the next stability properties of the Sinkhorn Operator:

\cite{eckstein2022quantitative}
    The Sinkhorn Operator is Lipschitz with respect to the cost function in wasserstein distance:
 \[
W_1\Big(S_c(\mu_,\nu),\,S_{c'}(\mu,\nu)\Big)
\;\le\; L_\varepsilon \|c-c'\|_\infty,
\]
where $L_\varepsilon$ is independent of $\mu,\nu$ .

\cite{nutz2023stability}
    Let $\mu_n \rightharpoonup \mu \in \cP(X)$ and $\nu_n \rightharpoonup \nu \in \cP(Y)$ then the density of the solution of the optimal transport problem converges uniformly $\frac{d\pi^\varepsilon}{d(\mu_n \otimes \nu_n)} \rightrightarrows \frac{d\pi^\varepsilon}{d(\mu \otimes \nu)}$.

\begin{lemma}[Approximation by Absolutely Continuous Plans]\label{lema_aprox_abs_cont}
Let \(X\) and \(Y\) be compact metric spaces, and let \(\pi \in \Pi(\mu, \nu)\) be a transport plan. Then there exists a sequence of measures \(\pi_k \in \Pi(\mu, \nu)\), \(k \in \mathbb{N}\), such that for all \(k\):
\begin{enumerate}
    \item \(\pi_k \ll \mu \otimes \nu\) (i.e., absolutely continuous with respect to the product measure),
    \item the density \(\frac{d\pi_k}{d(\mu \otimes \nu)}\) is bounded,
    \item \(\pi_k \rightharpoonup \pi\) as \(k \to \infty\) (weak convergence).
\end{enumerate}
Moreover, if \(\pi\) is a probability measure, then \(\pi_k\) are also probability measures.
\end{lemma}

\begin{proof}
The proof proceeds in several constructive steps.

\noindent\textbf{Step 1: Construction of partitions of \(X \times Y\).}

Since \(X\) and \(Y\) are compact, \(X \times Y\) is compact. For each \(k \in \mathbb{N}\), consider a finite partition \(P_k\) of \(X \times Y\) into measurable rectangles of the form \(A_i \times B_j\), where \(\{A_i\}_{i=1}^{N_k}\) is a measurable partition of \(X\) and \(\{B_j\}_{j=1}^{M_k}\) is a measurable partition of \(Y\), such that the diameter of each rectangle \(A_i \times B_j\) satisfies \(\text{diam}(A_i \times B_j) < \frac{1}{k}\).

This partition can be constructed as follows: consider on \(X \times Y\) the metric given by \(d((x,y),(\bar{x},\bar{y})) = d_X(x,\bar{x}) + d_Y(y,\bar{y})\). Then:
\begin{enumerate}
    \item For each \(k \in \mathbb{N}\), consider coverings of \(X\) and \(Y\) by closed balls of radius \(\frac{1}{4k}\).
    \item By compactness, extract finite subcoverings:
    \[
    X \subset \bigcup_{i=1}^{N_k} \overline{B}_X(x_i, \tfrac{1}{4k}), \quad 
    Y \subset \bigcup_{j=1}^{M_k} \overline{B}_Y(y_j, \tfrac{1}{4k}).
    \]
    \item Disjointify the coverings by defining inductively:
    \begin{align*}
    A_1 &= \overline{B}_X(x_1, \tfrac{1}{4k}), \\
    A_i &= \overline{B}_X(x_i, \tfrac{1}{4k}) \setminus \bigcup_{j=1}^{i-1} A_j, \quad i = 2,\ldots,N_k, \\
    B_1 &= \overline{B}_Y(y_1, \tfrac{1}{4k}), \\
    B_j &= \overline{B}_Y(y_j, \tfrac{1}{4k}) \setminus \bigcup_{\ell=1}^{j-1} B_\ell, \quad j = 2,\ldots,M_k.
    \end{align*}
    \item The rectangles \(A_i \times B_j\) form the partition \(P_k\) of \(X \times Y\). Note that \(\text{diam}_X(B(x,r)) \leq 2r\) and \(\text{diam}(A \times B) \leq \text{diam}_X(A) + \text{diam}_Y(B)\), from which it follows that \(\text{diam}(A_i \times B_j) < \frac{1}{k}\).
\end{enumerate}

\noindent\textbf{Step 2: Definition of \(\pi_k\).}

For each \(k \in \mathbb{N}\), define the measure \(\pi_k\) via its density with respect to \(\mu \otimes \nu\). Specifically, for each rectangle \(A_i \times B_j\) in \(P_k\), define:
\[
f_k(x,y) = \sum_{i=1}^{N_k} \sum_{j=1}^{M_k} \frac{\pi(A_i \times B_j)}{\mu(A_i)\nu(B_j)} \mathbf{1}_{A_i \times B_j}(x,y) \quad \text{for } (x,y) \in X \times Y,
\]
where we adopt the convention that if \(\mu(A_i) = 0\) or \(\nu(B_j) = 0\), then the corresponding term is zero (which is consistent since in that case \(\pi(A_i \times B_j) = 0\) because \(\pi\) is a transport plan).

Define then:
\[
d\pi_k(x,y) = f_k(x,y) \, d(\mu \otimes \nu)(x,y).
\]
Clearly, \(\pi_k\) is absolutely continuous with respect to \(\mu \otimes \nu\).

Note that if \(\pi\) is a probability measure, \(\pi_k\) is also a probability measure:
\begin{align*}
\pi_k(X \times Y) &= \int_{X \times Y} f_k(x,y) \, d(\mu \otimes \nu)(x,y) \\
&= \int_{X \times Y} \sum_{i=1}^{N_k} \sum_{j=1}^{M_k} \frac{\pi(A_i \times B_j)}{\mu(A_i)\nu(B_j)} \mathbf{1}_{A_i \times B_j}(x,y) \, d(\mu \otimes \nu)(x,y) \\
&= \sum_{i=1}^{N_k} \sum_{j=1}^{M_k} \frac{\pi(A_i \times B_j)}{\mu(A_i)\nu(B_j)} \int_{X \times Y} \mathbf{1}_{A_i \times B_j}(x,y) \, d(\mu \otimes \nu)(x,y).
\end{align*}
The integral \(\int_{X \times Y} \mathbf{1}_{A_i \times B_j}(x,y) \, d(\mu \otimes \nu)(x,y)\) equals \((\mu \otimes \nu)(A_i \times B_j) = \mu(A_i)\nu(B_j)\). Thus,
\[
\pi_k(X \times Y) = \sum_{i=1}^{N_k} \sum_{j=1}^{M_k} \frac{\pi(A_i \times B_j)}{\mu(A_i)\nu(B_j)} \cdot \mu(A_i)\nu(B_j) = \sum_{i=1}^{N_k} \sum_{j=1}^{M_k} \pi(A_i \times B_j).
\]
Since \(\{A_i \times B_j\}_{i,j}\) is a partition of \(X \times Y\), the sum above is \(\pi(X \times Y) = 1\). Therefore, \(\pi_k(X \times Y) = 1\), and \(\pi_k\) is a probability measure.

\noindent\textbf{Step 3: Verification that \(\pi_k\) is a transport plan.}

We must show that the marginals of \(\pi_k\) are exactly \(\mu\) and \(\nu\). Compute the marginal on \(X\) (the other is completely analogous). For any measurable set \(A \subset X\):
\begin{align*}
\pi_k(A \times Y) &= \int_{A \times Y} f_k(x,y) \, d(\mu \otimes \nu)(x,y) \\
&= \sum_{i=1}^{N_k} \sum_{j=1}^{M_k} \frac{\pi(A_i \times B_j)}{\mu(A_i)\nu(B_j)} \int_{A \times Y} \mathbf{1}_{A_i \times B_j}(x,y) \, d(\mu \otimes \nu)(x,y).
\end{align*}
Note that:
\[
\int_{A \times Y} \mathbf{1}_{A_i \times B_j}(x,y) \, d(\mu \otimes \nu)(x,y) = \mu(A \cap A_i)\nu(B_j),
\]
since \(\mathbf{1}_{A_i \times B_j}(x,y) = 1\) if and only if \(x \in A_i\) and \(y \in B_j\). Therefore:
\[
\pi_k(A \times Y) = \sum_{i=1}^{N_k} \sum_{j=1}^{M_k} \frac{\pi(A_i \times B_j)}{\mu(A_i)\nu(B_j)} \mu(A \cap A_i)\nu(B_j)
= \sum_{i=1}^{N_k} \sum_{j=1}^{M_k} \pi(A_i \times B_j) \frac{\mu(A \cap A_i)}{\mu(A_i)}.
\]
For each fixed \(i\), we have \(\sum_{j=1}^{M_k} \pi(A_i \times B_j) = \pi(A_i \times Y) = \mu(A_i)\), since the marginal of \(\pi\) on \(X\) is \(\mu\). Hence:
\[
\pi_k(A \times Y) = \sum_{i=1}^{N_k} \frac{\mu(A \cap A_i)}{\mu(A_i)} \mu(A_i) = \sum_{i=1}^{N_k} \mu(A \cap A_i) = \mu(A),
\]
where the last equality follows because the \(A_i\) form a partition of \(X\). Therefore, the first marginal of \(\pi_k\) is \(\mu\).

\noindent\textbf{Step 4: Weak convergence of \(\pi_k\) to \(\pi\).}

Let \(f: X \times Y \to \mathbb{R}\) be continuous. Since \(X \times Y\) is compact, \(f\) is uniformly continuous. Therefore, for every \(\delta > 0\), there exists \(k_0 \in \mathbb{N}\) such that for all \(k \geq k_0\) and for all \((x,y), (x', y')\) in the same rectangle \(A_i \times B_j\) of \(P_k\), we have:
\[
|f(x,y) - f(x', y')| < \delta. \tag{3.1}
\]
Now, for each rectangle \(A_i \times B_j\), choose a point \((x_i, y_j) \in A_i \times B_j\) and consider the integral:
\begin{align*} the 
\int f \, d\pi_k &= \int_{X \times Y} f(x,y) f_k(x,y) \, d(\mu \otimes \nu)(x,y) \tag{3.2} \\
&= \sum_{i=1}^{N_k} \sum_{j=1}^{M_k} \frac{\pi(A_i \times B_j)}{\mu(A_i)\nu(B_j)} \int_{A_i \times B_j} f(x,y) \, d(\mu \otimes \nu)(x,y).
\end{align*}
For each term in the sum, we have:
\begin{align*}
&\left| \int_{A_i \times B_j} f(x,y) \, d(\mu \otimes \nu)(x,y) - f(x_i, y_j) \mu(A_i)\nu(B_j) \right| \\
&\leq \int_{A_i \times B_j} |f(x,y) - f(x_i, y_j)| \, d(\mu \otimes \nu)(x,y) \leq \delta \mu(A_i)\nu(B_j).
\end{align*}
Therefore:
\[
\left| \frac{1}{\mu(A_i)\nu(B_j)} \int_{A_i \times B_j} (f(x,y) - f(x_i, y_j)) \, d(\mu \otimes \nu)(x,y) \right| \leq \delta. \tag{3.3}
\]
Combining (3.2) and (3.3):
\[
\left| \int f \, d\pi_k - \sum_{i=1}^{N_k} \sum_{j=1}^{M_k} \pi(A_i \times B_j) f(x_i, y_j) \right| \leq \sum_{i=1}^{N_k} \sum_{j=1}^{M_k} \pi(A_i \times B_j) \delta = \pi(X \times Y) \delta. \tag{3.4}
\]
On the other hand, consider the integral with respect to \(\pi\) and using (3.1):
\begin{align*}
&\left| \int f \, d\pi - \sum_{i=1}^{N_k} \sum_{j=1}^{M_k} f(x_i, y_j) \pi(A_i \times B_j) \right| \\
&\leq \sum_{i=1}^{N_k} \sum_{j=1}^{M_k} \int_{A_i \times B_j} |f(x,y) - f(x_i, y_j)| \, d\pi(x,y) \\
&\leq \pi(X \times Y) \delta. \tag{3.5}
\end{align*}
Combining inequalities (3.4) and (3.5) yields:
\[
\left| \int f \, d\pi_k - \int f \, d\pi \right| \leq 2\delta \pi(X \times Y).
\]
Since \(\delta > 0\) is arbitrary, we conclude that:
\[
\lim_{k \to \infty} \int f \, d\pi_k = \int f \, d\pi.
\]
Therefore, \(\pi_k \rightharpoonup \pi\) weakly.
\end{proof}

\begin{lemma}[Entropic Representation]\label{lema_representation_entropic}
Fix \((\mu, \nu)\) probability measures. Let \(\gamma \in \Pi(\mu, \nu)\) with density \(\rho := \frac{d\gamma}{d(\mu \otimes \nu)} > 0\) on \(\operatorname{supp}(\mu) \times \operatorname{supp}(\nu)\). Define \(c(x,y) := -\epsilon \log \rho(x,y)\). Then \(\gamma\) is the unique minimizer of the functional
\[
J(\gamma') := \int_{X \times Y} c \, d\gamma' + \epsilon \KL(\gamma' \parallel \mu \otimes \nu)
\]
among \(\gamma' \in \Pi(\mu, \nu)\).
\end{lemma}

\begin{proof}
It is convenient to write everything with respect to the reference measure \(\mu \otimes \nu\): if \(\gamma' \ll \mu \otimes \nu\) we write \(\rho' = \frac{d\gamma'}{d(\mu \otimes \nu)} \geq 0\). (If \(\gamma'\) had a singular part with respect to \(\mu \otimes \nu\), then \(\KL(\gamma' \parallel \mu \otimes \nu) = +\infty\) and it cannot be a minimizer; thus we only consider measures with densities \(\rho'\).)

Now compute \(J(\gamma')\):
\begin{align*}
J(\gamma') &= \int (-\epsilon \log \rho) \rho' \, d(\mu \otimes \nu) + \epsilon \int \rho' \log \rho' \, d(\mu \otimes \nu) \\
&= \epsilon \int \rho' \log \frac{\rho'}{\rho} \, d(\mu \otimes \nu) = \epsilon \int \frac{\rho'}{\rho} \log \frac{\rho'}{\rho} \, d\gamma \\
&= \epsilon \KL(\gamma' \parallel \gamma) \geq 0.
\end{align*}
We conclude that \(J(\gamma') = 0\) if and only if \(\rho' = \rho\) \(\mu \otimes \nu\)-almost everywhere. This completes the proof.
\end{proof}

\begin{lemma}[Gamma Convergence of Entropic Optimal Transport]\label{prop_gamma_convergence}
Let \(X\) and \(Y\) be compact metric spaces, \(\mu \in \mathcal{P}(X)\) and \(\nu \in \mathcal{P}(Y)\) fixed probability measures, and \(\varepsilon > 0\) a regularization parameter. Consider a sequence of cost functions \(c_n: X \times Y \to \mathbb{R}\) and a cost function \(c: X \times Y \to \mathbb{R}\) such that:
\begin{enumerate}
    \item \(c_n\) are continuous for all \(n \in \mathbb{N}\),
    \item There exists \(M > 0\) such that \(|c_n(x,y)| \leq M\) for all \(n \in \mathbb{N}\) and \((x,y) \in X \times Y\) (equiboundedness),
    \item \(c_n \to c\) in \(L^1(\mu \otimes \nu)\).
\end{enumerate}
For each \(n\), define the functional \(J_n: \Pi(\mu, \nu) \to \mathbb{R}\) by:
\[
J_n(\pi) = \int_{X \times Y} c_n \, d\pi + \varepsilon \KL(\pi \parallel \mu \otimes \nu),
\]
and define \(J: \Pi(\mu, \nu) \to \mathbb{R} \cup \{+\infty\}\) by:
\[
J(\pi) = \int_{X \times Y} c \, d\pi + \varepsilon \KL(\pi \parallel \mu \otimes \nu).
\]
Then, \(J_n \overset{\Gamma}{\to} J\) in \(\Pi(\mu, \nu)\) with the weak* topology.
\end{lemma}

\begin{proof}
We prove each part of the Gamma convergence separately.

\noindent\textbf{1. Liminf inequality.} Let \(\pi_n \rightharpoonup \pi\) in \(\Pi(\mu, \nu)\). We want to show that
\[
J(\pi) \leq \liminf_{n \to \infty} J_n(\pi_n).
\]
Without loss of generality, assume \(\liminf_{n \to \infty} J_n(\pi_n) < \infty\), otherwise the inequality is trivial. Then in particular:
\[
\KL(\pi_n \parallel \mu \otimes \nu) < \infty \quad \forall n \in \mathbb{N},
\]
which implies
\[
\pi_n \ll \mu \otimes \nu \quad \forall n \in \mathbb{N}.
\]
Moreover, since the KL divergence is lower semicontinuous:
\[
\KL(\pi \parallel \mu \otimes \nu) \leq \liminf_{n \to \infty} \KL(\pi_n \parallel \mu \otimes \nu) < \infty,
\]
so \(\pi\) is also absolutely continuous with respect to \(\mu \otimes \nu\). Let \(f_n, f\) be the densities of \(\pi_n, \pi\) with respect to the product measure. Using Vallée-Poussin's theorem \cite{brezis2010functional} for the function \(x \log(x)\), we conclude that \(\{f_n\}\) are uniformly integrable.

Since the KL divergence is lower semicontinuous, to prove the liminf inequality it suffices to show convergence of the linear term, i.e.:
\[
\int c_n f_n \, d(\mu \otimes \nu) \to \int c f \, d(\mu \otimes \nu).
\]
By Egorov's theorem, for any \(\delta > 0\), there exists a set \(E_\delta\) such that:
\[
c_n \to c \text{ in } E_\delta \quad \text{and} \quad \mu \otimes \nu(X\times Y \setminus E_\delta) < \delta.
\]
Then,
\[
\int_X c_n f_n \, d(\mu \otimes \nu) = \int_{E_\delta} c_n f_n \, d(\mu \otimes \nu) + \int_{X \times Y  \setminus E_\delta} c_n f_n \, d(\mu \otimes \nu).
\]
The second term can be made arbitrarily small because:
\[
\left| \int_{X \times Y  \setminus E_\delta} c_n f_n \, d(\mu \otimes \nu) \right| \leq \int_{X \times Y  \setminus E_\delta} |c_n| |f_n| \, d(\mu \otimes \nu).
\]
Using that the \(c_n\) are uniformly bounded and the \(f_n\) are uniformly integrable, taking an appropriate \(\delta\) yields:
\[
\left| \int_{X \times Y  \setminus E_\delta} c_n f_n \, d(\mu \otimes \nu) \right| \leq M \epsilon.
\]
For the first term,
\[
\int_{E_\delta} c_n f_n \, d(\mu \otimes \nu) = \int_{E_\delta} (c_n - c) f_n \, d(\mu \otimes \nu) + \int_{E_\delta} c f_n \, d(\mu \otimes \nu).
\]
Since \(\pi_n \rightharpoonup \pi\), \(f_n \rightharpoonup f\) weakly in \(L^1\), so \(\{f_n\}\) are bounded in \(L^1(\mu \otimes \nu)\). Thus,
\[
\int_{E_\delta} (c_n - c) f_n \, d(\mu \otimes \nu) \leq \|f_n\|_{L^1(\mu \otimes \nu)} \|c_n - c\|_{L^\infty(E_\delta)} \to 0,
\]
because \(c_n\) converges uniformly to \(c\) on \(E_\delta\). Meanwhile, since \(c\) is bounded, by weak convergence of \(f_n\):
\[
\int_{E_\delta} c f_n \, d(\mu \otimes \nu) \to \int_{E_\delta} c f \, d(\mu \otimes \nu).
\]
Putting everything together, we obtain:
\[
\liminf_{n \to \infty} \int_{X \times Y} c_n f_n \, d(\mu \otimes \nu) \geq \int_{E_\delta} c f \, d(\mu \otimes \nu) - M \varepsilon.
\]
Taking \(\delta, \varepsilon \to 0\), we conclude the desired inequality.

\noindent\textbf{2. Limsup inequality.} Let \(\pi \in \Pi(\mu, \nu)\). If \(J(\pi) = \infty\), the inequality is trivial. Suppose \(J(\pi) < \infty\), which implies \(\KL(\pi \parallel \mu \otimes \nu) < \infty\) and thus \(\pi \ll \mu \otimes \nu\). Take the constant sequence \(\pi_n = \pi\) for all \(n\), which clearly satisfies \(\pi_n \rightharpoonup \pi\). It suffices to show that:
\[
J_n(\pi) \to J(\pi).
\]
Note that the KL divergence term is the same for all \(n\). Therefore, we only need to analyze the linear term. Given a subsequence (still denoted by \((c_n)\)), since \(c_n \to c\) in \(L^1(\mu \otimes \nu)\), there exists a subsequence that converges \(\mu \otimes \nu\)-almost everywhere, and since \(\pi \ll \mu \otimes \nu\), it also converges \(\pi\)-almost everywhere. Moreover, using that the costs are uniformly bounded \(|c_n| \leq M\), we have \(|c_n - c| \leq 2M\). Then by the dominated convergence theorem:
\[
\int c_n \, d\pi \to \int c \, d\pi.
\]
This completes the proof.
\end{proof}

We would like to regularize the measures to facilitate calculations. With this in mind, we prove a regularization lemma.

\begin{lemma}[Regularization Lemma]\label{lema_regularizacion}
Let \(X\) and \(Y\) be compact metric spaces, \(\mu \in \mathcal{P}(X)\) and \(\nu \in \mathcal{P}(Y)\) probability measures, and \(\pi \in \Pi(\mu, \nu)\) a transport plan with density \(\rho \in L^\infty(\mu \otimes \nu)\) with respect to \(\mu \otimes \nu\). Then, for every \(\varepsilon > 0\), there exists a transport plan \(\pi_\varepsilon \in \Pi(\mu, \nu)\) with density \(\rho_\varepsilon \in C(X \times Y)\) continuous and strictly positive (\(\rho_\varepsilon > \eta > 0\)) such that \(W_p(\pi, \pi_\varepsilon) < \varepsilon\). Where $\eta$ is independent of $\mu,\nu$.
\end{lemma}

\begin{proof}
Let \(\varepsilon > 0\). The proof proceeds in three constructive stages.

\noindent\textbf{Stage 1: Regularization for strict positivity.}

Let \(0 < \delta < 1\). Define the regularized measure:
\[
\pi_\delta = (1 - \delta) \pi + \delta (\mu \otimes \nu).
\]
Since \(\Pi(\mu, \nu)\) is convex, it follows that \(\pi_\delta \in \Pi(\mu, \nu)\). Moreover, the density of \(\pi_\delta\) with respect to \(\mu \otimes \nu\) is given by \(\rho_\delta = (1 - \delta) \rho + \delta\), and since \(\rho \geq 0\) and \(\delta > 0\), we have \(\rho_\delta \geq \delta > 0\).

Clearly, \(\pi_\delta \to \pi\) as \(\delta \to 0\) in total variation norm. Therefore, there exists \(\delta_0 > 0\) such that:
\[
W_p(\pi, \pi_{\delta_0}) < \frac{\varepsilon}{3}.
\]
Define \(\pi_1 = \pi_{\delta_0}\) and denote \(\rho_1 = \rho_{\delta_0}\) its density.

\noindent\textbf{Stage 2: Approximation by continuous density.}

Consider the cost \(c = -\log(\rho_1)\). Since \(\rho_1\) is bounded and strictly positive, \(c\) is bounded.

By density of continuous functions in \(L^1(\mu \otimes \nu)\), given \(\delta > 0\), there exists a continuous function \(c_\delta \in C(X \times Y)\) such that:
\[
\|c - c_\delta\|_{L^1(\mu \otimes \nu)} < \delta.
\]
Also note that since $\rho_1$ is bounded from above independently of $\mu,\nu$, $c$ is bounded from below, so we can choose $c_\delta$ to be bounded from above too.

Now consider the entropic optimal transport problem:
\[
\min_{\pi \in \Pi(\mu, \nu)} \left\{ \int c_\delta \, d\pi + \KL(\pi \parallel \mu \otimes \nu) \right\}.
\]
By \cite{leonard2012schrodinger}, this problem has a unique solution \(\pi_\delta\) with density:
\[
\rho_\delta = u_\delta e^{-c_\delta} v_\delta,
\]
where \(u_\delta \in C(X)\) and \(v_\delta \in C(Y)\) are continuous positive functions. Therefore, \(\rho_\delta\) is continuous.

When \(\delta \to 0\), \(c_\delta \to c\) in \(L^1(\mu \otimes \nu)\). By Lemma \ref{prop_gamma_convergence} on stability of entropic solutions, \(\pi_\delta\) converges in \(W_p\) to the solution of the problem with cost \(c\), which is precisely \(\pi_1\).

Therefore, for \(\delta\) sufficiently small:
\[
W_p(\pi_1, \pi_\delta) < \frac{\varepsilon}{3}.
\]
Fix \(\delta_1\) with this property and call \(\pi_2 = \pi_{\delta_1}\).

\noindent\textbf{Stage 3: Final regularization for uniform positivity.}

Reasoning similarly to Stage 1, we can approximate \(\pi_2\) by another transport plan \(\pi_3 \in \Pi(\mu, \nu)\) such that it is absolutely continuous with respect to \(\mu \otimes \nu\) with strictly positive density, and from the construction it follows that the density of \(\pi_3\) is continuous since the density of \(\pi_2\) is continuous. Moreover, this plan \(\pi_3\) satisfies:
\[
W_p(\pi_2, \pi_3) < \frac{\varepsilon}{3}.
\]

Note that since $X,Y$ are bounded metric spaces we can choose $\pi_3$ by controlling the total variation norm we have a lower bound independent of $\mu,\nu$. Combining the three estimates:
\[
W_p(\pi, \pi_3) \leq W_p(\pi, \pi_1) + W_p(\pi_1, \pi_2) + W_p(\pi_2, \pi_3) < \frac{\varepsilon}{3} + \frac{\varepsilon}{3} + \frac{\varepsilon}{3} = \varepsilon.
\]
The plan \(\pi_3\) has density \(\rho_3\) continuous and strictly positive (\(\rho_3 \geq \eta > 0\)), which completes the proof.
\end{proof}

\begin{corollary} \label{coro_aprox_puntual}
    Let $\mu \in \cP(X)$ and $\nu \in \cP(Y)$ each transport plan $\pi \in \Pi(\mu,\nu)$ can be approximated by a solution of the entropic optimal transport problem $S_c(\mu,\nu)$.
\end{corollary}

\begin{proof}
    Combining Lemma \ref{lema_aprox_abs_cont} with Lemma \ref{lema_regularizacion}, we obtain that for any transport plan, we can approximate it by one with
continuous density with respect to the product measure. Applying Lemma \ref{lema_representation_entropic} to the density, we conclude the proof.
\end{proof}

We already proved pointwise approximation by the Sinkhorn operator, our next step will be to pass pass to a uniform approximation, varying the cost function continuously with the measures.

\begin{theorem}[Approximation of Coupling Systems by Sinkhorn Operators]\label{teo_aprox_uniforme}
Let \(X\) and \(Y\) be compact metric spaces, and let \(F: \mathcal{P}(X) \times \mathcal{P}(Y) \to \mathcal{P}(X \times Y)\) be a Coupling System. Then, for every \(\varepsilon > 0\), and $p \geq 1$, there exists a Sinkhorn operator with a continuous cost $c\colon \mathcal{P}(X) \times \mathcal{P}(Y) \to C(X \times Y)$ such that
\[
\sup_{(\mu, \nu) \in \mathcal{P}(X) \times \mathcal{P}(Y)} W_p \left( S_{c(\mu, \nu)}(\mu, \nu), F(\mu, \nu) \right) < \varepsilon.
\]
\end{theorem}

\begin{proof}
The proof is organized in six main steps.

\noindent\textbf{Step 1: Equivalence of distances and reduction to \(W_1\).}

Without loss of generality we can work with $W_1$. Let $\gamma$ be a $W_1$-optimal coupling. For $p\ge1$, $\int d^p\,d\gamma\le D^{p-1}\int d\,d\gamma$, since $X\times Y$ is compact. So
$W_p\le D^{1-1/p}W_1^{1/p}$. So choosing $\varepsilon$ small enough, we get the desired control.

\noindent\textbf{Step 2: Definition of the multifunction \(\Gamma\).}

Let \(\eta > 0\) be the constant from the previous lemma \ref{lema_regularizacion},

For each pair \((\mu, \nu) \in \mathcal{P}(X) \times \mathcal{P}(Y)\), define the set:
\[
\Gamma(\mu, \nu) := \left\{ f \in C(X \times Y) \ \middle| \ 
\begin{aligned}
&(1) \ f > \eta \text{ on } X \times Y, \\
&(2) \ f \cdot (\mu \otimes \nu) \in \Pi(\mu, \nu), \\
&(3) \ f \cdot (\mu \otimes \nu) \in B_{W_1}\left(F(\mu, \nu), \frac{\varepsilon}{2}\right)
\end{aligned} \right\}.
\]

\noindent\textbf{Step 3: Properties of \(\Gamma\).}

We show that \(\Gamma(\mu, \nu)\) satisfies the following properties:
\begin{itemize}
    \item \textbf{Non-empty:} By the continuity and approximation lemmas (Lemmas \ref{lema_aprox_abs_cont} and \ref{lema_regularizacion}), for each \((\mu, \nu)\) there exists at least one continuous density \(f \geq \eta\) that approximates \(F(\mu, \nu)\) within \(W_1\) distance less than \(\frac{\varepsilon}{2}\). Therefore, \(\Gamma(\mu, \nu) \neq \emptyset\) for all \((\mu, \nu)\).
    
    \item \textbf{Convex:} Let \(f_1, f_2 \in \Gamma(\mu, \nu)\) and \(t \in [0, 1]\). Define \(f_t = t f_1 + (1 - t) f_2\). Verify each condition:
    \begin{itemize}
        \item \textbf{Positivity:} \(f_t = t f_1 + (1 - t) f_2 > t \eta + (1 - t) \eta = \eta\).
        \item \textbf{Normalization:} \(\int f_t \, d(\mu \otimes \nu) = t \int f_1 \, d(\mu \otimes \nu) + (1 - t) \int f_2 \, d(\mu \otimes \nu) = t + (1 - t) = 1\).
        \item \textbf{Marginals:} For any measurable \(A \subset X\),
        \[
        \int_{A \times Y} f_t \, d(\mu \otimes \nu) = t \mu(A) + (1 - t) \mu(A) = \mu(A),
        \]
        and similarly for \(\nu\). Hence, \(f_t \cdot (\mu \otimes \nu) \in \Pi(\mu, \nu)\).
    \end{itemize}
    
    \item \textbf{Distance to \(F(\mu, \nu)\):} Let \(\pi_1 = f_1 \cdot (\mu \otimes \nu)\), \(\pi_2 = f_2 \cdot (\mu \otimes \nu)\), and \(\pi_t = f_t \cdot (\mu \otimes \nu) = t \pi_1 + (1 - t) \pi_2\). Using the Kantorovich-Rubinstein dual formulation it is easy to prove that the ball is convex, so we have
    \[
    W_1(\pi_t, F(\mu, \nu)) \leq t W_1(\pi_1, F(\mu, \nu)) + (1 - t) W_1(\pi_2, F(\mu, \nu)) < \frac{\varepsilon}{2}.
    \]
    Therefore, \(f_t \in \Gamma(\mu, \nu)\) and \(\Gamma(\mu, \nu)\) is convex.
\end{itemize}

Moreover, as shown in Proposition \ref{hemi_gamma_copuling}, the multifunction \(\Gamma\) is lower hemicontinuous.

\noindent\textbf{Step 4: Approximate continuous selection.}

Since \(\Gamma(\mu, \nu)\) is convex, non-empty, and \(\Gamma\) is lower hemicontinuous, we can apply the approximate selection theorem \cite{deutsch1983continuous}. Then, there exists then an approximate continuous selection
\[
s: \mathcal{P}(X) \times \mathcal{P}(Y) \to C(X \times Y)
\]
such that
\[
\inf_{t \in \Gamma(\mu, \nu)} \|s(\mu, \nu) - t\|_\infty \leq \frac{\varepsilon}{2 L_{\text{sink}} L_{\log}},
\]
where:
\begin{itemize}
    \item \(L_{\log}\) is the Lipschitz constant of the transformation \(-\log: f \mapsto -\log f\) for functions greater than some \(f > \eta > 0\):
    \[
    |-\log s(x,y) + \log t(x,y)| \leq \frac{1}{\eta} |s(x,y) - t(x,y)|,
    \]
    i.e., \(\|-\log(s) + \log(t)\|_\infty \leq \frac{1}{\eta} \|s - t\|_\infty\).
    
    \item \(L_{\text{sink}}\) is the Lipschitz constant of the Sinkhorn operator.
\end{itemize}
Note that, since \(s(\mu, \nu)\) is uniformly close to strictly positive functions (greater than \(\eta\)), for fixed \((\mu, \nu)\), \(s(\mu, \nu)\) is also strictly positive.

\noindent\textbf{Step 5: Definition of the cost.}

Define the cost \(c\) as
\[
c(\mu, \nu)(x, y) := -\log s(\mu, \nu)(x, y).
\]
This function remains continuous because \(s(\mu, \nu)\) is continuous and we have shown that applying the logarithm to such functions preserves continuity.

\noindent\textbf{Step 6: Final verification of the approximation.}

For each \((\mu, \nu)\), let \(t \in \Gamma(\mu, \nu)\) such that \(\|s(\mu, \nu) - t\|_\infty \leq \frac{\varepsilon}{2 L_{\text{sink}} L_{\log}}\). Then,
\begin{align*}
W_1\left(F(\mu, \nu), S_{c(\mu, \nu)}(\mu, \nu)\right) 
&\leq W_1\left(F(\mu, \nu), S_{-\log t}(\mu, \nu)\right) \\
&\quad + W_1\left(S_{-\log t}(\mu, \nu), S_{-\log s(\mu, \nu)}(\mu, \nu)\right).
\end{align*}
Analyze each term separately:
\begin{itemize}
    \item \textbf{First term:} By construction of \(\Gamma\), and since \(S_{-\log t}(\mu, \nu) = t \cdot (\mu \otimes \nu)\) (Lemma \ref{lema_representation_entropic}), we have
    \[
    W_1\left(F(\mu, \nu), S_{-\log t}(\mu, \nu)\right) < \frac{\varepsilon}{2}.
    \]
    
    \item \textbf{Second term:} By the Lipschitz continuity of the Sinkhorn operator \cite{eckstein2022quantitative} and the logarithmic transformation, we have
    \[
    W_1\left(S_{-\log t}(\mu, \nu), S_{-\log s(\mu, \nu)}(\mu, \nu)\right) \leq L_{\text{sink}} L_{\log} \|t - s(\mu, \nu)\|_\infty < \frac{\varepsilon}{2}.
    \]
\end{itemize}
Combining both estimates, we obtain
\[
W_1\left(F(\mu, \nu), S_{c(\mu, \nu)}(\mu, \nu)\right) < \frac{\varepsilon}{2} + \frac{\varepsilon}{2} = \varepsilon.
\]
This completes the proof of the theorem.
\end{proof}

\begin{proposition}[Lower Hemicontinuity of \(\Gamma\)] \label{hemi_gamma_copuling}
The multifunction
\[
\Gamma: \mathcal{P}(X) \times \mathcal{P}(Y) \to C(X \times Y)
\]
is lower hemicontinuous. That is, for every \((\mu_0, \nu_0)\), \(s_0 \in \Gamma(\mu_0, \nu_0)\) and every sequence \((\mu_n, \nu_n) \to (\mu_0, \nu_0)\) in the weak topology, there exists a sequence \(s_n \in \Gamma(\mu_n, \nu_n)\) such that \(s_n \to s_0\) uniformly.
\end{proposition} 

\begin{proof}
The proof is organized in four constructive steps.

\noindent\textbf{Step 1: Construction of the sequence \(s_n\).}

Let \(s_0 \in \Gamma(\mu_0, \nu_0)\) with \(s_0 > \eta > 0\). Consider the cost function \(c = -\log(s_0)\). Since \(s_0\) is continuous and strictly positive, \(c\) is continuous and bounded.

By \cite{leonard2012schrodinger} on existence and uniqueness for the entropic problem, for each \(n\) there exist continuous functions \(u_n, v_n > 0\) such that:
\[
s_n(x, y) := u_n(x) s_0(x, y) v_n(y)
\]
satisfies \(s_n \cdot (\mu_n \otimes \nu_n) \in \Pi(\mu_n, \nu_n)\).

For \(\mu_0, \nu_0\), the solution of the entropic problem with cost \(c\) is unique and corresponds to \(s_0 (\mu_0 \otimes \nu_0)\), with potentials \(u_0 = 1\) and \(v_0 = 1\).

\noindent\textbf{Step 2: Uniform convergence of \(s_n\) to \(s_0\).}

By \cite{nutz2023stability} on uniform convergence of Schrödinger potentials, we have:
\[
\|u_n v_n - 1\|_\infty \to 0 \quad \text{as } n \to \infty.
\]
Then, for the convergence of \(s_n\):
\[
\|s_n - s_0\|_\infty = \|s_0 u_n v_n - s_0\|_\infty \leq \|s_0\|_\infty \|u_n v_n - 1\|_\infty \to 0.
\]

Moreover, by compactness and continuity, \(s_0\) attains its minimum:
\[
\min_{(x,y) \in X \times Y} s_0(x, y) > \eta > 0.
\]
For \(n\) sufficiently large:
\[
s_n(x, y) = s_0(x, y) u_n(x) v_n(y) > \eta \quad \text{for all } (x, y) \in X \times Y.
\]

\noindent\textbf{Step 3: Weak convergence of the measures.}

Consider the convergence of the measures \(s_n \cdot (\mu_n \otimes \nu_n)\) to \(s_0 \cdot (\mu_0 \otimes \nu_0)\). For any test function \(\varphi \in C(X \times Y)\), analyze:
\begin{align*}
I_n(\varphi) &:= \int \varphi \, d(s_n \cdot (\mu_n \otimes \nu_n)) - \int \varphi \, d(s_0 \cdot (\mu_0 \otimes \nu_0)) \\
&= \int \varphi (s_n - s_0) \, d(\mu_n \otimes \nu_n) + \int \varphi s_0 \, d(\mu_n \otimes \nu_n) - \int \varphi s_0 \, d(\mu_0 \otimes \nu_0).
\end{align*}
\begin{itemize}
    \item For \(A_n(\varphi) = \int \varphi (s_n - s_0) \, d(\mu_n \otimes \nu_n)\):
    \[
    |A_n(\varphi)| \leq \|\varphi\|_\infty \|s_n - s_0\|_\infty \to 0.
    \]
    
    \item For \(B_n(\varphi) = \int \varphi s_0 \, d(\mu_n \otimes \nu_n) - \int \varphi s_0 \, d(\mu_0 \otimes \nu_0)\): since \(\varphi s_0\) is continuous and \((\mu_n, \nu_n) \rightharpoonup (\mu_0, \nu_0)\) weakly, we have \(B_n(\varphi) \to 0\).
\end{itemize}
Therefore, \(s_n \cdot (\mu_n \otimes \nu_n) \rightharpoonup s_0 \cdot (\mu_0 \otimes \nu_0)\).

\noindent\textbf{Step 4: Control of the \(W_1\) distance.}

Verify that for \(n\) sufficiently large, \(s_n \in B_{W_1}\left(F(\mu_n, \nu_n), \frac{\varepsilon}{2}\right)\). Let \(D := \text{diam}(X \times Y) < \infty\) and define:
\[
d_0 := W_1(s_0 \cdot (\mu_0 \otimes \nu_0), F(\mu_0, \nu_0)) < \frac{\varepsilon}{2},
\]
where the strict inequality follows from \(s_0 \in \Gamma(\mu_0, \nu_0)\). Let \(\delta_0 := \frac{\varepsilon}{2} - d_0 > 0\).

By the triangle inequality:
\begin{align*}
&W_1(s_n \cdot (\mu_n \otimes \nu_n), F(\mu_n, \nu_n)) 
\leq \underbrace{W_1(s_n \cdot (\mu_n \otimes \nu_n), s_0 \cdot (\mu_n \otimes \nu_n))}_{T_{1,n}}  + d_0 + \underbrace{W_1(F(\mu_0, \nu_0), F(\mu_n, \nu_n))}_{T_{2,n}}.
\end{align*}
Analyze each term:
\begin{itemize}
    \item \(T_{1,n}  \to 0\) by step 3.
    \item \(T_{2,n} \to 0\) by continuity of \(F\).
\end{itemize}
Consequently, there exists \(N\) such that for all \(n \geq N\): \(T_{1,n}, T_{2,n} \leq \delta_0/2\). Then:
\[
W_1(s_n \cdot (\mu_n \otimes \nu_n), F(\mu_n, \nu_n)) < \frac{\delta_0}{2}  + d_0 + \frac{\delta_0}{2} = \frac{\varepsilon}{2}.
\]

By construction, \(s_n\) is continuous, \(s_n > \eta\), \(s_n(\mu_n \otimes \nu_n) \in \Pi(\mu_n, \nu_n)\), and \(s_n \in \Gamma(\mu_n, \nu_n)\) for \(n\) sufficiently large. This concludes the proof.
\end{proof}

\begin{theorem}[Universal Approximation of Coupling Systems by Coupling Transformers]\label{teo_principal}
Let \(X \subset \mathbb R^d\) and \(Y \subset \mathbb R^{d'}\) be compact sets, and let \(F: \mathcal{P}(X) \times \mathcal{P}(Y) \to \mathcal{P}(X \times Y)\) be a Coupling System. Then, for every \(\varepsilon > 0\), and $p\geq1$, there exists a Coupling Transformer architecture \(\mathcal{T}^*\) such that:
\[
\sup_{(\mu, \nu) \in \mathcal{P}(X) \times \mathcal{P}(Y)} W_p\left( \mathcal{T}^*(\mu, \nu), F(\mu, \nu) \right) < \varepsilon.
\]
\end{theorem}

\begin{proof}
We separate the proof into six steps.

\noindent\textbf{Step 1: Equivalence of distances and reduction to \(W_1\).}

Without loss of generality we can work with $W_1$. Let $\gamma$ be a $W_1$-optimal coupling. For $p\ge1$, $\int d^p\,d\gamma\le D^{p-1}\int d\,d\gamma$, since $X\times Y$ is compact. So
$W_p\le D^{1-1/p}W_1^{1/p}$. So choosing $\varepsilon$ small enough, we get the desired control.

\medskip
\noindent\textbf{Step 2: Reduction to continuous costs via entropic optimal transport.}
By Theorem \ref{teo_aprox_uniforme} for every \(\varepsilon > 0\) there exists a continuous function
\[
c: \mathcal{P}(X) \times \mathcal{P}(Y) \to C(X \times Y)
\]
such that
\[
\sup_{(\mu,\nu)} W_1(S_{c(\mu,\nu)}(\mu, \nu), F(\mu,\nu)) < \frac{\varepsilon}{2},
\]
where \(S_{c(\mu,\nu)}\) is the Sinkhorn plan associated with the cost \(c(\mu,\nu)\).

\medskip
\noindent\textbf{Step 3: Error control via the Lipschitz continuity of Sinkhorn.}
Let \(L_{\text{Sink}}\) be the Lipschitz constant of the operator \(c \mapsto S_c\) with respect to the supremum norm on bounded costs, as stated in \cite{eckstein2022quantitative}. Then
\[
W_1(S_c(\mu, \nu), S_{c'}(\mu, \nu)) \leq L_{\text{Sink}} \|c-c'\|_\infty,
\]
where \(L_{\text{Sink}}\) is independent of \((\mu, \nu)\in \mathcal{P}(X) \times \mathcal{P}(Y)\).

\medskip
\noindent\textbf{Step 4: Approximation of the cost by inner products.}
Note that the space of continuous functions \( C(\mathcal P(X) \times \mathcal P(Y),C(X \times Y))\) is equivalent to \(C(\mathcal P(X) \times \mathcal P(Y) \times X \times Y)\). 

Moreover, the space \(\mathcal{P}(X) \times \mathcal{P}(Y) \times X \times Y\) is compact, because \(\mathcal{P}(X)\) and \(\mathcal{P}(Y)\) are compact (in the weak-* topology) and \(X\) and \(Y\) are compact. Consider the algebra of functions of the form:
\[
f(\mu,\nu,x,y) = \sum_{k=1}^d G_k(\mu,x) H_k(\nu,y),
\]
where \(G_k: \mathcal{P}(X) \times X \to \mathbb{R}\) and \(H_k: \mathcal{P}(Y) \times Y \to \mathbb{R}\) are continuous. This algebra:
\begin{itemize}
\item Contains constants: if \(d = 1\) and \(G_1 \equiv 1\), \(H_1 \equiv c\), then \(f \equiv c\).
\item Separates points: for any pair of distinct points \((\mu,\nu,x,y) \neq (\mu',\nu',x',y')\), there exist functions \(G_k\) and \(H_k\) that distinguish these points.
\end{itemize}
By the Stone–Weierstrass theorem, this algebra is dense in \(C(\mathcal{P}(X) \times \mathcal{P}(Y) \times X \times Y)\).

Since \(c \in C(\mathcal{P}(X) \times \mathcal{P}(Y) \times X \times Y)\), for any \(\delta > 0\) there exist continuous functions \(G_k, H_k\) for \(k = 1,\ldots,d\) such that:
\[
\left| c(\mu,\nu)(x,y) - \sum_{k=1}^d G_k(\mu,x) H_k(\nu,y) \right| < \delta \quad \forall (\mu,\nu,x,y).
\]
Define:
\begin{align*}
G(\mu,x) &= (G_1(\mu,x), \ldots, G_d(\mu,x)) \in \mathbb{R}^d, \\
H(\nu,y) &= (H_1(\nu,y), \ldots, H_d(\nu,y)) \in \mathbb{R}^d,
\end{align*}
and the approximated cost:
\[
c'(\mu,\nu)(x,y) = \langle G(\mu,x), H(\nu,y) \rangle = \sum_{k=1}^d G_k(\mu,x) H_k(\nu,y).
\]
Then:
\[
\|c - c'\|_\infty < \delta.
\]

\medskip
\noindent\textbf{Step 5: Approximation of \(G\) and \(H\) by Transformers.}
Now, by the Universal Approximation Theorem for in-context Transformers
\cite{furuya2024transformers}, the continuous functions \(G\) and \(H\), defined
on a compact domain, can be approximated uniformly by two Transformer encoders
\(Q^*\) and \(K^*\). Using linearity of the Euclidean inner product, we write
\begin{align*}
\langle Q^*(\mu,x), K^*(\nu,y) \rangle - \langle G(\mu,x), H(\nu,y) \rangle
&= \langle Q^*(\mu,x) - G(\mu,x),\, K^*(\nu,y) \rangle \\
&\quad + \langle G(\mu,x),\, K^*(\nu,y) - H(\nu,y) \rangle.
\end{align*}
Applying the triangle inequality and the Cauchy--Schwarz inequality in
\(\mathbb{R}^d\), we obtain
\begin{align*}
\big|\langle Q^*(\mu,x), K^*(\nu,y) \rangle - \langle G(\mu,x), H(\nu,y) \rangle\big|
&\le \|Q^*(\mu,x) - G(\mu,x)\|_2 \, \|K^*(\nu,y)\|_2 \\
&\quad + \|G(\mu,x)\|_2 \, \|K^*(\nu,y) - H(\nu,y)\|_2.
\end{align*}
Since
\(
\|K^*(\nu,y)\|_2 \le \|H(\nu,y)\|_2 + \|K^*(\nu,y) - H(\nu,y)\|_2,
\)
taking the supremum over the compact domain yields
\begin{align*}
\sup \big|\langle Q^*, K^* \rangle - \langle G, H \rangle\big|
&\le \|Q^* - G\|_\infty \, \|H\|_\infty \\
&\quad + \big(\|Q^* - G\|_\infty + \|G\|_\infty\big)\, \|K^* - H\|_\infty.
\end{align*}
Since \(G\) and \(H\) are fixed and uniformly bounded, for any \(\delta > 0\) the
right-hand side can be made smaller than \(\delta\) by choosing the
approximations sufficiently accurate.

\medskip
\noindent\textbf{Step 6: Putting everything together.}
Define the Contextual Transformer \(\mathcal{T}^* = S_{c^*}\) with 
\[
c^*(\mu, \nu)(x,y) = \langle Q^*(\mu,x), K^*(\nu,y)\rangle.
\]
Then \(\mathcal{T}^*\) has the Sinkhorn-Transformer architecture. Using the Lipschitz bound on the Sinkhorn operator and the triangle inequality, we obtain:
\begin{align*}
\sup_{(\mu,\nu)} W_1(F(\mu,\nu), \mathcal{T}^*(\mu,\nu)) &\leq \sup_{(\mu,\nu)} W_1(F(\mu,\nu), S_{c(\mu,\nu)}(\mu,\nu)) \\
&\quad + \sup_{(\mu,\nu)} W_1(S_{c(\mu,\nu)}(\mu,\nu), \mathcal{T}^*(\mu,\nu))
\end{align*}
By choosing \(\delta < \frac{\varepsilon}{2 L_{\text{Sink}}}\), we finally get
\[
\sup_{(\mu,\nu)} W_1(F(\mu,\nu), \mathcal{T}^*(\mu,\nu)) < \varepsilon,
\]
which concludes the proof.

\end{proof}

Moreover we have that the class of coupling systems is exactly the class of functions able to approximate by coupling transformers, as we can see in the next proposition. 

\begin{proposition}\label{prop_contra_coupling}
    Let $\mathcal T_n \to \mathcal F$ be a sequence of Coupling Transformers converging uniformly in Wasserstein distance, then $F$ is a coupling system.
\end{proposition}
\begin{proof}
    It follows because $\Pi(\mu,\nu)$ is closed under Wasserstein distance (\cite{villani2009optimal}) and the uniform convergence preserves continuity in the limit.
\end{proof}
%

\section{Conditional Systems}
\label{app:kernel_systems}

In this section we prove the main theorem for conditional systems. The scheme is the same
one we followed for coupling systems --- approximate an arbitrary element by an absolutely
continuous one, regularize it, select a regularization continuously in the data, and
finally factorize the resulting exponent by Stone--Weierstrass --- but every density must
now be \emph{row-normalized}: normalized in the second variable for each fixed value of the
first. This is exactly the constraint that a softmax imposes, and the operator below makes
it explicit.

Throughout the section we keep the conditional notation, and we use a tilde for the
associated joint measure. Thus $F(\nu\mid\mu)$ denotes the conditional measure produced by
the system $F$ on the contexts $(\mu,\nu)$, and
\[
\widetilde F(\nu\mid\mu):=\mu\otimes F(\nu\mid\mu)\in\Pi_1(\mu;\nu)
\]
is the corresponding measure on $X\times Y$. The tilde is used in the same way for any other
object producing conditional measures; in particular, conditional transformers are written in
calligraphic type and
\[
\widetilde{\mathcal T}(\nu\mid\mu):=\mu\otimes\mathcal T(\nu\mid\mu) .
\]
We also use the following notation.

\begin{definition}[Softmax joint operator]\label{def:F_softmax_joint}
Let $\mu\in\mathcal{P}(X)$, $\nu\in\mathcal{P}(Y)$ and let $\alpha:X\times Y\to\mathbb{R}$ be
bounded and Borel. Put
\[
Z_\alpha(x):=\int_Y e^{\alpha(x,y)}\,\nu(dy),
\qquad\text{so that}\qquad
e^{-\|\alpha\|_\infty}\le Z_\alpha\le e^{\|\alpha\|_\infty},
\]
and define
\[
\mathcal{S}_{\nu\mid\mu}(\alpha):=\mu\otimes\mathrm{Softmax}_\nu(\alpha)\in\mathcal{P}(X\times Y),
\qquad
\frac{d\,\mathcal{S}_{\nu\mid\mu}(\alpha)}{d(\mu\otimes\nu)}(x,y)=\frac{e^{\alpha(x,y)}}{Z_\alpha(x)} .
\]
The density is row-normalized.
\end{definition}

Two features of this operator will be used repeatedly. First, the normalization is carried
out separately for each $x$, so that the density produced is automatically row-normalized:
these are precisely the measures a conditional transformer can output. Second, the
normalizing constant $Z_\alpha$ is bounded away from $0$ and from $\infty$ in terms of
$\|\alpha\|_\infty$ alone, which is what makes the operator stable under perturbations of
the exponent.

All the continuity statements below involve two measures moving simultaneously, together
with integrands that move with them. The following elementary lemma isolates that situation
once and for all.

\begin{lemma}[Uniform convergence of integrals]\label{lem:F_uniform}
Let $X,Y$ be compact metric spaces, $\nu_n\rightharpoonup\nu_0$ in $\mathcal{P}(Y)$, and let
$h_n,h_0:X\times Y\to\mathbb{R}$ be bounded Borel functions with $h_0$ continuous and
$\|h_n-h_0\|_\infty\to0$. Put $\psi_n(x):=\int_Y h_n(x,y)\,\nu_n(dy)$ for $n\ge0$. Then:
\begin{enumerate}
\item[(a)] $\psi_0\in C(X)$ and $\|\psi_n-\psi_0\|_\infty\to0$.
\item[(b)] If moreover $\mu_n\rightharpoonup\mu_0$ in $\mathcal{P}(X)$, then
  $\int_X\psi_n\,d\mu_n\to\int_X\psi_0\,d\mu_0$. In particular, taking
  $h_n\equiv h_0=\varphi$ with $\varphi\in C(X\times Y)$, one obtains
  $\mu_n\otimes\nu_n\rightharpoonup\mu_0\otimes\nu_0$.
\end{enumerate}
\end{lemma}

\begin{proof}
\emph{(a) The limit is continuous.}
$\psi_0$ is continuous because $h_0$ is uniformly continuous on the compact set $X\times Y$.

\emph{Splitting the error.}
For every $x$,
\[
|\psi_n(x)-\psi_0(x)|\le\|h_n-h_0\|_\infty+\Big|\int_Y h_0(x,y)\,\nu_n(dy)-\int_Y h_0(x,y)\,\nu_0(dy)\Big| .
\]
The first term tends to $0$ by hypothesis, uniformly in $x$; only the second one, which no
longer involves $h_n$, has to be estimated.

\emph{Reduction to products.}
Fix $\theta>0$. Finite sums $\sum_{i=1}^N a_i(x)b_i(y)$ with $a_i\in C(X)$, $b_i\in C(Y)$ form a
subalgebra of $C(X\times Y)$ that contains the constants and separates points, hence is dense by
Stone--Weierstrass. Choose one with $\|h_0-\sum_i a_ib_i\|_\infty<\theta$. Then, for every $x$,
\[
\Big|\int_Y h_0(x,\cdot)\,d(\nu_n-\nu_0)\Big|
\le2\theta+\sum_{i=1}^N\|a_i\|_\infty\,\Big|\int_Y b_i\,d\nu_n-\int_Y b_i\,d\nu_0\Big|
\xrightarrow[n\to\infty]{}2\theta ,
\]
and the bound is uniform in $x$, since the functions $a_i$ enter only through their sup norms.
As $\theta$ is arbitrary, $\|\psi_n-\psi_0\|_\infty\to0$.

\emph{(b) Passing to the outer integral.}
\[
\Big|\int\psi_n\,d\mu_n-\int\psi_0\,d\mu_0\Big|
\le\|\psi_n-\psi_0\|_\infty+\Big|\int\psi_0\,d\mu_n-\int\psi_0\,d\mu_0\Big|\longrightarrow0 ,
\]
the first term by (a) and the second because $\psi_0$ is continuous and $\mu_n\rightharpoonup\mu_0$.
\end{proof}

We now collect the properties of the softmax operator that play, in the conditional setting,
the role that Lemma~\ref{lema_representation_entropic} and
Lemma~\ref{prop_gamma_convergence} played in the coupling one: it represents exactly
the positive row-normalized densities, it is stable in total variation under uniform
perturbations of the exponent, and it depends continuously on the pair of contexts
$(\mu,\nu)$.

\begin{lemma}\label{lem:F_softmax}
Let $X,Y$ be compact metric spaces and $\mu\in\mathcal{P}(X)$, $\nu\in\mathcal{P}(Y)$.
\begin{enumerate}
\item[(i)] \textup{(Identity.)} If $f\in C(X\times Y)$, $f>0$ and $f$ is row-normalized, then
  $\mathcal{S}_{\nu\mid\mu}(\log f)=f\cdot(\mu\otimes\nu)$.
\item[(ii)] \textup{(Stability.)} For bounded Borel $\alpha,\alpha'$ with $t:=\|\alpha-\alpha'\|_\infty$,
  \[
  \big\|\mathcal{S}_{\nu\mid\mu}(\alpha)-\mathcal{S}_{\nu\mid\mu}(\alpha')\big\|_{TV}\le\frac{e^{2t}-1}{2}.
  \]
\item[(iii)] \textup{(Continuity in the contexts.)} If
  $a:\mathcal{P}(X)\times\mathcal{P}(Y)\times X\times Y\to\mathbb{R}$ is continuous, then
  $(\mu,\nu)\mapsto\mathcal{S}_{\nu\mid\mu}\big(a(\mu,\nu,\cdot,\cdot)\big)$ is continuous from
  $\mathcal{P}(X)\times\mathcal{P}(Y)$ to $\mathcal{P}(X\times Y)$.
\end{enumerate}
\end{lemma}

\begin{proof}
\emph{(i) The softmax is the identity on row-normalized densities.}
Since $f$ is row-normalized, $Z_{\log f}(x)=\int_Y f(x,y)\,\nu(dy)=1$ for every $x$, so the
density of $\mathcal{S}_{\nu\mid\mu}(\log f)$ with respect to $\mu\otimes\nu$ is $e^{\log f}/1=f$.

\emph{(ii) A pointwise comparison of the two rows.}
Fix $x$ and put $p:=e^{\alpha(x,\cdot)}/Z_\alpha(x)$ and $p':=e^{\alpha'(x,\cdot)}/Z_{\alpha'}(x)$.
Since $e^{-t}\le e^{\alpha-\alpha'}\le e^{t}$ pointwise, integrating in $y$ gives
$e^{-t}\le Z_{\alpha'}(x)/Z_\alpha(x)\le e^{t}$; multiplying the two estimates,
$e^{-2t}\le p/p'\le e^{2t}$, whence
\[
|p-p'|=p'\,\Big|\frac{p}{p'}-1\Big|\le p'\,(e^{2t}-1) .
\]

\emph{Integration.}
Integrating in $y$ gives $\int_Y|p-p'|\,d\nu\le e^{2t}-1$ for every $x$, and integrating in $x$
against $\mu$ gives the same bound for the $L^1(\mu\otimes\nu)$ distance of the two densities.
As the total variation distance is half of that $L^1$ distance, the bound follows.

\emph{(iii) Reduction to a ratio of integrals.}
Let $(\mu_n,\nu_n)\to(\mu_0,\nu_0)$ and $a_n:=a(\mu_n,\nu_n,\cdot,\cdot)$. Since $a$ is
uniformly continuous on the compact space $\mathcal{P}(X)\times\mathcal{P}(Y)\times X\times Y$,
$a_n\to a_0$ uniformly on $X\times Y$; let $A:=\|a\|_\infty$. Fix $\varphi\in C(X\times Y)$ and put
\[
N_n(x):=\int_Y\varphi(x,y)e^{a_n(x,y)}\,\nu_n(dy),
\qquad
Z_n(x):=\int_Y e^{a_n(x,y)}\,\nu_n(dy),
\]
so that $\int\varphi\,d\mathcal{S}_{\nu_n\mid\mu_n}(a_n)=\int_X(N_n/Z_n)\,d\mu_n$.

\emph{Conclusion.}
Lemma~\ref{lem:F_uniform}(a), applied to $h_n=\varphi e^{a_n}$ and to $h_n=e^{a_n}$, gives
$N_n\to N_0$ and $Z_n\to Z_0$ uniformly on $X$, with $N_0,Z_0$ continuous and $Z_n\ge e^{-A}$.
The uniform lower bound on the denominators turns this into $N_n/Z_n\to N_0/Z_0$ uniformly,
with $N_0/Z_0$ continuous, and the argument of Lemma~\ref{lem:F_uniform}(b) gives
$\int(N_n/Z_n)\,d\mu_n\to\int(N_0/Z_0)\,d\mu_0$.
\end{proof}

We turn now to the approximation scheme itself. The first step is to replace an arbitrary
conditional measure by one which is absolutely continuous with respect to the product
measure; it is the analogue of Lemma~\ref{lema_aprox_abs_cont}. The idea is to smear each
conditional law over balls of radius $r$, which costs at most $r$ in $W_1$ and produces a
bounded density.

\begin{lemma}[Ball averaging]\label{lem:F_ball}
Let $X,Y$ be compact metric spaces and $\mu\in\mathcal{P}(X)$, $\nu\in\mathcal{P}(Y)$,
$\pi\in\Pi_1(\mu;\nu)$ and $r>0$. Then there is a bounded Borel row-normalized function
$f_r\ge0$ on $X\times Y$ with
\[
W_1\big(f_r\cdot(\mu\otimes\nu),\,\pi\big)\le r ,
\qquad
0\le f_r\le\frac1{m_r(\nu)},\quad m_r(\nu):=\inf_{y'\in\operatorname{supp}\nu}\nu\big(B(y',r)\big)>0 .
\]
In particular $f_r\cdot(\mu\otimes\nu)\in\Pi_1(\mu;\nu)$.
\end{lemma}

\begin{proof}
\emph{The normalizing quantity $m_r(\nu)$ is positive.}
The map $y'\mapsto\nu(B(y',r))$ is lower semicontinuous: if $y_n\to y'$ then
$\mathbf 1_{B(y',r)}\le\liminf_n\mathbf 1_{B(y_n,r)}$ pointwise, and Fatou's lemma gives
$\nu(B(y',r))\le\liminf_n\nu(B(y_n,r))$. It is strictly positive on $\operatorname{supp}\nu$ by
definition of the support, and a strictly positive lower semicontinuous function on the
compact set $\operatorname{supp}\nu$ attains its infimum; hence $m_r(\nu)>0$. Lower
semicontinuity also shows that $y'\mapsto\nu(B(y',r))$ is Borel, so that the formulas below
define Borel functions.

\emph{Construction of the density.}
Let $K$ be a Markov kernel with $\pi(dx,dy')=\mu(dx)K(x,dy')$. Without loss of generality we
may assume $K(x,\operatorname{supp}\nu)=1$ for \emph{every} $x\in X$. Define the Borel function
$g:Y\times Y\to\mathbb{R}$ and the candidate density $f_r$ by
\[
g(y,y'):=\frac{\mathbf 1_{\{d(y,y')<r\}}\,\mathbf 1_{\operatorname{supp}\nu}(y')}{\nu(B(y',r))}
\quad(g:=0\text{ if }y'\notin\operatorname{supp}\nu),
\qquad
f_r(x,y):=\int_Y g(y,y')\,K(x,dy') ,
\]
so that $0\le f_r\le1/m_r(\nu)$: the kernel $K(x,\cdot)$ is a probability measure and the
integrand is bounded by $1/m_r(\nu)$.

\emph{The density is row-normalized.}
For every $x$, Tonelli's theorem and $K(x,\operatorname{supp}\nu)=1$ give
\[
\int_Y f_r(x,y)\,\nu(dy)
=\int_Y\Big[\frac{\int_Y\mathbf 1_{\{d(y,y')<r\}}\,\nu(dy)}{\nu(B(y',r))}\Big]\mathbf 1_{\operatorname{supp}\nu}(y')\,K(x,dy')
=\int_Y\mathbf 1_{\operatorname{supp}\nu}(y')\,K(x,dy')=1 ,
\]
since the bracket equals $1$ for every $y'\in\operatorname{supp}\nu$. Hence $f_r$ is
row-normalized and $\pi_r:=f_r\cdot(\mu\otimes\nu)\in\Pi_1(\mu;\nu)$.

\emph{An explicit coupling of $\pi$ and $\pi_r$.}
We claim that $W_1(\pi_r,\pi)\le r$, and we prove it by exhibiting a coupling that achieves
that transportation cost. For $y'\in\operatorname{supp}\nu$ put
$\nu_{y',r}:=\nu(\cdot\cap B(y',r))/\nu(B(y',r))\in\mathcal{P}(Y)$ (and $\nu_{y',r}:=\nu$
otherwise; this choice is irrelevant since $K(x)$ lives on $\operatorname{supp}\nu$). The map
$y'\mapsto\nu_{y',r}$ is a Markov kernel because
$y'\mapsto\nu(B\cap B(y',r))=\int\mathbf 1_B(y)\mathbf 1_{\{d(y,y')<r\}}\nu(dy)$ is Borel for
each Borel $B$. Define $\gamma\in\mathcal{P}\big((X\times Y)\times(X\times Y)\big)$ by
\[
\gamma(dx,dy',d\tilde x,dy):=\mu(dx)\,K(x,dy')\,\delta_x(d\tilde x)\,\nu_{y',r}(dy) .
\]

\emph{Its marginals are the right ones.}
Integrating out $(\tilde x,y)$ gives the first marginal $\mu(dx)K(x,dy')=\pi$. For a bounded
Borel $\psi$ on $X\times Y$, Tonelli's theorem gives
\[
\int\psi(\tilde x,y)\,d\gamma
=\int_X\int_Y\psi(x,y)\Big[\int_Y g(y,y')\,K(x,dy')\Big]\nu(dy)\,\mu(dx)
=\int\psi\,f_r\,d(\mu\otimes\nu),
\]
so the second marginal is $\pi_r$.

\emph{The cost of the coupling.}
On the support of $\gamma$ one has $\tilde x=x$ and $d(y,y')<r$, hence
$d\big((x,y'),(\tilde x,y)\big)=|y-y'|<r$ $\gamma$-a.s., and therefore
$W_1(\pi_r,\pi)\le\int d\,d\gamma\le r$.
\end{proof}

The density produced by the previous lemma is merely bounded and Borel, whereas the softmax
operator produces continuous and strictly positive densities. The next lemma, the analogue of
Lemma~\ref{lema_regularizacion}, upgrades it in three moves --- mixing with the product
measure to gain positivity, mollifying to gain continuity, and renormalizing to restore the
row condition --- while keeping the loss in $W_1$ under control. The point that makes the
lemma usable later is that the resulting lower bound $\eta(\varepsilon)$ is the same for all
$(\mu,\nu)$ and all $\pi$.

\begin{lemma}[Regularization]\label{lem:F_reg}
Let $X,Y$ be compact metric spaces. For $\varepsilon>0$ put
$\eta(\varepsilon):=\min\{1/2,\ \varepsilon/(4D)\}$. Let $\mu\in\mathcal{P}(X)$,
$\nu\in\mathcal{P}(Y)$ and $\pi\in\Pi_1(\mu;\nu)$. Then there is a row-normalized
$f\in C(X\times Y)$ with
\[
f>\eta(\varepsilon)\ \text{ on }X\times Y
\qquad\text{and}\qquad
W_1\big(f\cdot(\mu\otimes\nu),\,\pi\big)<\varepsilon .
\]
The constant $\eta(\varepsilon)$ depends only on $\varepsilon$ and $D$; not on $(\mu,\nu)$ or $\pi$.
\end{lemma}

\begin{proof}
Write $\eta:=\eta(\varepsilon)$, and recall that on a space of diameter $D$ one has
$W_1\le D\,\|\cdot\|_{TV}$, which is how each of the moves below is paid for.

\emph{Step 1: a bounded density.}
Lemma~\ref{lem:F_ball} with $r:=\varepsilon/8$ gives a bounded Borel row-normalized $f_1$ with
$0\le f_1\le M:=1/m_r(\nu)$ and $W_1\big(f_1\cdot(\mu\otimes\nu),\pi\big)\le\varepsilon/8$.

\emph{Step 2: positivity.}
Define $f_2:=(1-\eta)f_1+\eta$, a convex combination of $f_1$ with the constant density $1$;
it is still row-normalized, and
\[
\big\|f_2\cdot(\mu\otimes\nu)-f_1\cdot(\mu\otimes\nu)\big\|_{TV}
=\eta\,\big\|\mu\otimes\nu-f_1\cdot(\mu\otimes\nu)\big\|_{TV}\le\eta ,
\]
so that $W_1\big(f_2\cdot(\mu\otimes\nu),f_1\cdot(\mu\otimes\nu)\big)\le D\eta\le\varepsilon/4$.

\emph{Step 3: continuity, and restoring the row condition.}
Let $\theta:=\varepsilon/(4D)$. Since $\mu\otimes\nu$ is a finite Borel measure on a compact
metric space, $C(X\times Y)$ is dense in $L^1(\mu\otimes\nu)$; choose $g_0\in C(X\times Y)$ with
$\|g_0-f_2\|_{L^1(\mu\otimes\nu)}<\eta$ and let $g:=\min\{\max\{g_0,\eta\},M\}$. Clipping to
$[\eta,M]$ is $1$-Lipschitz and fixes $f_2$, so $g\in C(X\times Y)$, $\eta\le g\le M$ and
$\|g-f_2\|_{L^1(\mu\otimes\nu)}<\eta$. Mollifying destroys the row condition, so we restore it
by dividing by the row masses: put
\[
Z_g(x):=\int_Y g(x,y)\,\nu(dy)\in[\eta,M],\qquad\tilde g(x,y):=\frac{g(x,y)}{Z_g(x)} .
\]
Here $Z_g$ is continuous by uniform continuity of $g$, so $\tilde g\in C(X\times Y)$; moreover
$\tilde g$ is row-normalized and $\tilde g\ge\eta/M>0$. The renormalization is cheap: using
$Z_{f_2}\equiv1$,
\[
\|\tilde g-g\|_{L^1(\mu\otimes\nu)}
=\int_X\Big(\int_Y g(x,y)\,\nu(dy)\Big)\Big|\frac1{Z_g(x)}-1\Big|\,\mu(dx)
=\int_X|Z_{f_2}(x)-Z_g(x)|\,\mu(dx)\le\|g-f_2\|_{L^1(\mu\otimes\nu)} ,
\]
hence $\|\tilde g-f_2\|_{L^1(\mu\otimes\nu)}<2\eta$ and
$W_1\big(\tilde g\cdot(\mu\otimes\nu),f_2\cdot(\mu\otimes\nu)\big)\le D\eta=\varepsilon/4$.

\emph{Step 4: a uniform lower bound.}
The bound $\tilde g\ge\eta/M$ obtained above still depends on $(\mu,\nu)$ through $M$, so we
mix once more: let $f:=(1-\eta)\tilde g+\eta$. Then $f\in C(X\times Y)$ is row-normalized,
$f\ge\eta+(1-\eta)\eta/M>\eta$, and exactly as in Step~2,
$W_1\big(f\cdot(\mu\otimes\nu),\tilde g\cdot(\mu\otimes\nu)\big)\le D\eta\le\varepsilon/4$.

\emph{Conclusion.}
By the triangle inequality,
$W_1\big(f\cdot(\mu\otimes\nu),\pi\big)\le\varepsilon/8+\varepsilon/4+\varepsilon/4+\varepsilon/4<\varepsilon$.
\end{proof}

Combining the last two lemmas with the identity of Lemma~\ref{lem:F_softmax}(i) --- which
turns a positive row-normalized density into the image of its logarithm under the softmax
--- we obtain, exactly as in Corollary~\ref{coro_aprox_puntual}, a pointwise approximation
statement.

\begin{corollary}[Approximation by softmax operators]\label{aprox_softmax_puntual}
Let $X,Y$ be compact metric spaces, $\varepsilon>0$, $\mu\in\mathcal{P}(X)$,
$\nu\in\mathcal{P}(Y)$ and $\pi\in\Pi_1(\mu;\nu)$. There exists $\alpha\in C(X\times Y)$ such that
\[
W_1\big(\mathcal{S}_{\nu\mid\mu}(\alpha),\pi\big)<\varepsilon .
\]
\end{corollary}

The exponent $\alpha$ given by the corollary depends on $(\mu,\nu)$ in an uncontrolled way.
The next theorem, the analogue of Theorem~\ref{teo_aprox_uniforme}, upgrades the pointwise
statement to a uniform one by choosing the exponent continuously in the contexts: the set of
admissible regularizations defines a set-valued map to which an approximate selection theorem
applies.

\begin{theorem}[Approximation of conditional systems by softmax operators]\label{thm:F_D6}
Let $X,Y$ be compact metric spaces, $F$ a conditional system, $\varepsilon>0$ and $p\ge1$.
There is a continuous map
$c:\mathcal{P}(X)\times\mathcal{P}(Y)\to(C(X\times Y),\|\cdot\|_\infty)$ such that
\[
\sup_{(\mu,\nu)\in\mathcal{P}(X)\times\mathcal{P}(Y)}
W_p\Big(\mathcal{S}_{\nu\mid\mu}\big(c(\mu,\nu)\big),\ \widetilde F(\nu\mid\mu)\Big)<\varepsilon .
\]
\end{theorem}

\begin{proof}
\emph{Reduction to $p=1$.}
Since $X\times Y$ is compact of diameter $D$, any coupling $\gamma$ satisfies
$\int d^p\,d\gamma\le D^{p-1}\int d\,d\gamma$; applying this to a $W_1$-optimal $\gamma$ gives
$W_p\le D^{1-1/p}W_1^{1/p}$. Hence a $W_1$-estimate with a sufficiently small tolerance yields
the desired $W_p$-estimate, and we may work with $W_1$ throughout.

\emph{The set-valued map of admissible densities.}
Let $\eta:=\eta(\varepsilon/2)$ be the constant of Lemma~\ref{lem:F_reg} and define, for
$(\mu,\nu)\in\mathcal{P}(X)\times\mathcal{P}(Y)$,
\[
\Gamma(\mu,\nu):=\left\{f\in C(X\times Y):
\begin{aligned}
&(1)\ f>\eta\ \text{on }X\times Y;\\
&(2)\ f\ \text{is row-normalized w.r.t.\ }\nu;\\
&(3)\ W_1\big(f\cdot(\mu\otimes\nu),\widetilde F(\nu\mid\mu)\big)<\tfrac{\varepsilon}{2}
\end{aligned}
\right\}.
\]
Thus $\Gamma$ takes values in the normed space $(C(X\times Y),\|\cdot\|_\infty)$ and is defined
on the compact metric space $\mathcal{P}(X)\times\mathcal{P}(Y)$. Note that the threshold $\eta$
in (1) is the same for every $(\mu,\nu)$, which is what makes the following estimates uniform.

\emph{A continuous approximate selection.}
Let
\[
\kappa:=\min\Big\{\frac{\eta}{2},\ \frac{\eta}{4}\log\Big(1+\frac{\varepsilon}{D}\Big)\Big\}>0 .
\]
By Proposition~\ref{prop:F_Gamma} below, $\Gamma$ has non-empty convex values and is lower
semicontinuous, so it satisfies the hypotheses of the approximate selection theorem of
\cite{deutsch1983continuous}. Applying it with tolerance $\kappa/2$ produces a continuous
$s:\mathcal{P}(X)\times\mathcal{P}(Y)\to C(X\times Y)$ with
$d\big(s(\mu,\nu),\Gamma(\mu,\nu)\big)\le\kappa/2<\kappa$ for all $(\mu,\nu)$. Fix $(\mu,\nu)$ and
choose $t\in\Gamma(\mu,\nu)$ with $\|s(\mu,\nu)-t\|_\infty<\kappa$.

\emph{The selection is positive.}
Since $t>\eta$ and $\kappa\le\eta/2$, we get $s(\mu,\nu)>\eta-\kappa\ge\eta/2>0$, so its
logarithm is well defined.

\emph{Definition and continuity of $c$.}
Put $c(\mu,\nu):=\log s(\mu,\nu)\in C(X\times Y)$. Since $s\ge\eta/2$ everywhere and $\log$ is
$(2/\eta)$-Lipschitz on $[\eta/2,\infty)$,
\[
\|c(\mu,\nu)-c(\mu',\nu')\|_\infty\le\frac2\eta\,\|s(\mu,\nu)-s(\mu',\nu')\|_\infty ,
\]
so $c$ is continuous.

\emph{Estimate.}
Both $s(\mu,\nu)$ and $t$ are $\ge\eta/2$, so the same Lipschitz bound gives
\begin{equation}\label{eq:F_logbound}
\|\log s(\mu,\nu)-\log t\|_\infty\le\frac2\eta\|s(\mu,\nu)-t\|_\infty<\frac{2\kappa}{\eta}
\le\frac12\log\Big(1+\frac\varepsilon D\Big) .
\end{equation}
By Lemma~\ref{lem:F_softmax}(i), $\mathcal{S}_{\nu\mid\mu}(\log t)=t\cdot(\mu\otimes\nu)$, since $t$
is continuous, positive and row-normalized. Hence, by Lemma~\ref{lem:F_softmax}(ii) applied with
\eqref{eq:F_logbound}, together with $W_1\le D\|\cdot\|_{TV}$ and condition (3) in the definition
of $\Gamma$,
\begin{align*}
W_1\big(\mathcal{S}_{\nu\mid\mu}(c(\mu,\nu)),\,\widetilde F(\nu\mid\mu)\big)
&\le W_1\big(\mathcal{S}_{\nu\mid\mu}(\log s(\mu,\nu)),\,\mathcal{S}_{\nu\mid\mu}(\log t)\big)
   +W_1\big(t\cdot(\mu\otimes\nu),\,\widetilde F(\nu\mid\mu)\big) \\[4pt]
&< D\,\frac{e^{\log(1+\varepsilon/D)}-1}{2}+\frac\varepsilon2
= \varepsilon .
\end{align*}
The bound is uniform in $(\mu,\nu)$ because $\kappa$ and $\eta$ do not depend on $(\mu,\nu)$.
\end{proof}

It remains to check the two hypotheses of the selection theorem used above, for the map
$\Gamma$ defined in the previous proof.

\begin{proposition}\label{prop:F_Gamma}
$\Gamma$ has non-empty convex values and is lower semicontinuous.
\end{proposition}

\begin{proof}
\emph{Non-emptiness.}
$\widetilde F(\nu\mid\mu)\in\Pi_1(\mu;\nu)$, so Lemma~\ref{lem:F_reg} applied with tolerance
$\varepsilon/2$ produces an element of $\Gamma(\mu,\nu)$.

\emph{Convexity.}
Let $f_1,f_2\in\Gamma(\mu,\nu)$, $t\in[0,1]$ and $f_t:=tf_1+(1-t)f_2$. Conditions (1) and (2) are
preserved by convex combinations. For (3), write $f_t\cdot(\mu\otimes\nu)=t\pi_1+(1-t)\pi_2$ with
$\pi_i:=f_i\cdot(\mu\otimes\nu)$, and observe that
\[
\pi\mapsto W_1\big(\pi,\widetilde F(\nu\mid\mu)\big)
=\sup_{\mathrm{Lip}(\varphi)\le1}\int\varphi\,d\big(\pi-\widetilde F(\nu\mid\mu)\big)
\]
is a supremum of affine functions of $\pi$, hence convex; therefore
\[
W_1\big(f_t\cdot(\mu\otimes\nu),\widetilde F(\nu\mid\mu)\big)
\le tW_1\big(\pi_1,\widetilde F(\nu\mid\mu)\big)+(1-t)W_1\big(\pi_2,\widetilde F(\nu\mid\mu)\big)<\varepsilon/2 .
\]

\emph{Lower semicontinuity: the candidates.}
Let $(\mu_n,\nu_n)\to(\mu_0,\nu_0)$ and $s_0\in\Gamma(\mu_0,\nu_0)$; we must produce
$s_n\in\Gamma(\mu_n,\nu_n)$ with $s_n\to s_0$. The only condition that $s_0$ itself may fail for
$(\mu_n,\nu_n)$ is the row normalization, so we simply renormalize it. Put
$Z_n(x):=\int_Y s_0(x,y)\,\nu_n(dy)$. By Lemma~\ref{lem:F_uniform}(a) (with $h_n\equiv h_0=s_0$)
and condition (2) for $s_0$, $Z_n\to Z_0\equiv1$ uniformly on $X$. Choose $n_1$ such that
$\|Z_n-1\|_\infty\le1/2$ for $n\ge n_1$ and, for such $n$, define
\[
s_n(x,y):=\frac{s_0(x,y)}{Z_n(x)} .
\]
Then $s_n\in C(X\times Y)$, $s_n$ is row-normalized with respect to $\nu_n$, and
\[
\|s_n-s_0\|_\infty\le\|s_0\|_\infty\Big\|\frac1{Z_n}-1\Big\|_\infty\le2\|s_0\|_\infty\|Z_n-1\|_\infty\to0 .
\]

\emph{Condition (1).}
$s_0$ is continuous and $>\eta$ on the compact set $X\times Y$, so $\min s_0>\eta$; as
$s_n\to s_0$ uniformly, $s_n>\eta$ for all large $n$.

\emph{Condition (3).}
Put $\pi_n:=s_n\cdot(\mu_n\otimes\nu_n)$ and $\pi_0:=s_0\cdot(\mu_0\otimes\nu_0)$. For
$\varphi\in C(X\times Y)$,
\[
\Big|\int\varphi\,d\pi_n-\int\varphi\,d\pi_0\Big|
\le\|\varphi\|_\infty\|s_n-s_0\|_\infty
+\Big|\int\varphi s_0\,d(\mu_n\otimes\nu_n)-\int\varphi s_0\,d(\mu_0\otimes\nu_0)\Big|\to0
\]
by the uniform convergence just proved and by weak convergence of the product measures. Hence
$\pi_n\rightharpoonup\pi_0$, i.e.\ $T_{1,n}:=W_1(\pi_n,\pi_0)\to0$. By continuity of
$\widetilde F$, $T_{2,n}:=W_1\big(\widetilde F(\nu_n\mid\mu_n),\widetilde F(\nu_0\mid\mu_0)\big)\to0$.
Writing $d_0:=W_1\big(\pi_0,\widetilde F(\nu_0\mid\mu_0)\big)<\varepsilon/2$, which is condition (3)
for $s_0$, the triangle inequality gives
\[
W_1\big(\pi_n,\widetilde F(\nu_n\mid\mu_n)\big)\le T_{1,n}+d_0+T_{2,n}<\varepsilon/2
\qquad\text{for sufficiently large }n .
\]

\emph{Conclusion.}
There is $n_2\ge n_1$ with $s_n\in\Gamma(\mu_n,\nu_n)$ for all $n\ge n_2$. For $n<n_2$ let $s_n$ be
an arbitrary element of the non-empty set $\Gamma(\mu_n,\nu_n)$. Then $s_n\in\Gamma(\mu_n,\nu_n)$
for all $n$ and $s_n\to s_0$ uniformly.
\end{proof}

We can now prove the main approximation theorem. The exponent $c(\mu,\nu)$ of
Theorem~\ref{thm:F_D6} is an arbitrary continuous function of the two contexts and the two
points; to realize it by a transformer block it must be written as an inner product of a
query, depending only on $(\mu,x)$, and a key, depending only on $(\nu,y)$. This separation is
achieved by Stone--Weierstrass, and the resulting factors are then replaced by encoders using
the universal approximation hypothesis.

\begin{theorem}[Approximation of conditional systems by conditional transformers]\label{thm:F_main}
Let $X\subset\mathbb{R}^d$ and $Y\subset\mathbb{R}^{d'}$ be compact, $F$ a conditional system,
$\varepsilon>0$ and $p\ge1$. There exists a conditional transformer $\mathcal{T}^{*}$ such that
\[
\sup_{(\mu,\nu)\in\mathcal{P}(X)\times\mathcal{P}(Y)}
W_p\big(\mu\otimes\mathcal{T}^{*}(\nu\mid\mu),\, \widetilde F(\nu\mid\mu)\big)<\varepsilon .
\]
\end{theorem}

\begin{proof}
\emph{Reduction to $p=1$.}
As in Theorem~\ref{thm:F_D6}: since $X\times Y$ is compact of diameter $D$, any coupling $\gamma$
satisfies $\int d^p\,d\gamma\le D^{p-1}\int d\,d\gamma$, so that $W_p\le D^{1-1/p}W_1^{1/p}$ and a
sufficiently small $W_1$-bound suffices. We therefore work with $W_1$.

\emph{Step 1 (continuous affinity).}
Theorem~\ref{thm:F_D6} with tolerance $\varepsilon/2$ gives a continuous
$c:\mathcal{P}(X)\times\mathcal{P}(Y)\to C(X\times Y)$ with
$\sup_{(\mu,\nu)}W_1\big(\mathcal{S}_{\nu\mid\mu}(c(\mu,\nu)),\widetilde F(\nu\mid\mu)\big)<\varepsilon/2$.

\emph{Step 2 (joint continuity).}
The function $\hat c(\mu,\nu,x,y):=c(\mu,\nu)(x,y)$ is continuous on the compact space
$\mathcal{P}(X)\times\mathcal{P}(Y)\times X\times Y$, since
\[
|\hat c(\mu,\nu,x,y)-\hat c(\mu_0,\nu_0,x_0,y_0)|
\le\|c(\mu,\nu)-c(\mu_0,\nu_0)\|_\infty+|c(\mu_0,\nu_0)(x,y)-c(\mu_0,\nu_0)(x_0,y_0)| ,
\]
where the first term is small by continuity of $c$ and the second one by continuity of the fixed
function $c(\mu_0,\nu_0)$.

\emph{Step 3 (Stone--Weierstrass).}
Let $\mathcal{A}$ be the set of finite sums $\sum_{k=1}^m G_k(\mu,x)H_k(\nu,y)$ with
$G_k\in C(\mathcal{P}(X)\times X)$ and $H_k\in C(\mathcal{P}(Y)\times Y)$; this is the class of
affinities a transformer block can produce. It is a subalgebra of
$C(\mathcal{P}(X)\times\mathcal{P}(Y)\times X\times Y)$ (a product of two such sums is again such a
sum), it contains the constants, and it separates points: if $(\mu,x)\ne(\mu',x')$, pick $G$ with
$G(\mu,x)\ne G(\mu',x')$ (continuous functions separate points of the compact metric space
$\mathcal{P}(X)\times X$) and $H\equiv1$; if $(\nu,y)\ne(\nu',y')$, argue symmetrically. Set
\[
\delta:=\tfrac14\log\big(1+\varepsilon/D\big)>0 .
\]
By the Stone--Weierstrass theorem there are continuous
$G=(G_1,\dots,G_m):\mathcal{P}(X)\times X\to\mathbb{R}^m$ and
$H=(H_1,\dots,H_m):\mathcal{P}(Y)\times Y\to\mathbb{R}^m$ with
$\|\hat c-\langle G,H\rangle\|_\infty<\delta/2$.

\emph{Step 4 (encoders).}
Write $\|G\|_\infty:=\sup|G(\mu,x)|$ and similarly for $H$ (Euclidean norms in $\mathbb{R}^m$).
Choose $\rho>0$ with $\rho\,(\|H\|_\infty+\|G\|_\infty+\rho)<\delta/2$ and, by (UA), encoders
$Q^*,K^*$ with $\|Q^*-G\|_\infty\le\rho$ and $\|K^*-H\|_\infty\le\rho$. Writing
\[
\langle Q^*,K^*\rangle-\langle G,H\rangle=\langle Q^*-G,K^*\rangle+\langle G,K^*-H\rangle ,
\]
the Cauchy--Schwarz inequality and $\|K^*\|_\infty\le\|H\|_\infty+\rho$ give
\[
\big\|\langle Q^*,K^*\rangle-\langle G,H\rangle\big\|_\infty\le\rho(\|H\|_\infty+\rho)+\|G\|_\infty\rho<\delta/2 .
\]
Hence $\alpha^*:=\langle Q^*,K^*\rangle$ satisfies $\|\hat c-\alpha^*\|_\infty<\delta$.

\emph{Step 5 (conclusion).}
By Lemma~\ref{lem:F_softmax}(ii), and since $W_1\le D\|\cdot\|_{TV}$, for every $(\mu,\nu)$,
\[
W_1\big(\mathcal{S}_{\nu\mid\mu}(\alpha^*_{(\mu,\nu)}),\mathcal{S}_{\nu\mid\mu}(c(\mu,\nu))\big)
\le D\,\frac{e^{2\delta}-1}{2}=D\,\frac{\sqrt{1+\varepsilon/D}-1}{2}<\frac{\varepsilon}{2},
\]
using $\sqrt{1+u}-1<u$ for $u>0$. Since $\mu\otimes\mathcal{T}^{*}(\nu\mid\mu)=\mathcal{S}_{\nu\mid\mu}(\alpha^*_{(\mu,\nu)})$,
the triangle inequality with Step~1 gives $W_1\big(\mu\otimes\mathcal{T}^{*}(\nu\mid\mu),\widetilde F(\nu\mid\mu)\big)<\varepsilon$,
uniformly in $(\mu,\nu)$.
\end{proof}

The approximation goes in one direction only if the class of limits is not larger than the class
of conditional systems. The converse inclusion holds as well: conditional systems are closed
under uniform limits of conditional transformers.

\begin{proposition}\label{prop:F_closed}
Let $\mathcal T_n\to\mathcal F$ be a sequence of conditional transformers converging uniformly in
Wasserstein distance. Then $\mathcal F$ is a conditional system.
\end{proposition}

\begin{proof}
We check the three requirements in the definition.

\emph{Continuity.}
The affinity $\alpha_{(\mu,\nu)}(x,y)=\langle Q(\mu,x),K(\nu,y)\rangle$ is jointly continuous, so
$\widetilde{\mathcal T}(\nu\mid\mu)=\mathcal{S}_{\nu\mid\mu}(\alpha_{(\mu,\nu)})$ is continuous by
Lemma~\ref{lem:F_softmax}(iii); uniform convergence then preserves continuity in the limit.

\emph{First marginal.}
$\operatorname{proj}_X$ is weakly continuous (\cite{villani2009optimal}), so the first marginal
passes to the limit.

\emph{Support.}
Since $U:=X\times(Y\setminus\operatorname{supp}\nu)$ is open, the Portmanteau theorem gives
$\widetilde{\mathcal F}(\nu\mid\mu)(U)\le\liminf_n\widetilde{\mathcal T}_n(\nu\mid\mu)(U)=0$, the terms on the right
vanishing because each $\widetilde{\mathcal T}_n(\nu\mid\mu)$ is absolutely continuous with respect to
$\mu\otimes\nu$.
\end{proof}

Finally we derive the next corollary.

\begin{corollary}[Approximation of conditional matrices]\label{cor:F_practical}
Let $X\subset\mathbb{R}^{d}$ and $Y\subset\mathbb{R}^{d'}$ be compact. For any conditional system
$F$ and any $\varepsilon,\delta>0$ there exists a Conditional Transformer Block $\mathcal{T}^{*}$,
depending only on $F$, $\varepsilon$ and $\delta$, such that for \emph{every} pair of texts
$(t_1,t_2)$ whose token embeddings are separated by at least $\delta$ within each text,
\[
\max_{i,j}\big|F_{ij}(t_2\mid t_1)-\mathcal{T}^{*}_{ij}(t_2\mid t_1)\big|<\varepsilon ,
\]
where $F_{ij}(t_2\mid t_1)$ and $\mathcal{T}^{*}_{ij}(t_2\mid t_1)$ denote the entries of the
corresponding attention matrices.
\end{corollary}

\begin{proof}
Write $\omega_d$ for the volume of the unit ball of $\mathbb{R}^{d}$ and set
\[
N:=\Big(1+\frac{2\operatorname{diam}X}{\delta}\Big)^{d},
\qquad
\varepsilon':=\frac{\varepsilon\,\delta}{2N}>0 .
\]
Let $\mathcal{T}^{*}$ be a Conditional Transformer Block given by Theorem~\ref{thm:F_main} applied
to the system $F$ with tolerance $\varepsilon'$, so that
\begin{equation}\label{eq:unif}
\sup_{(\mu,\nu)\in\mathcal{P}(X)\times\mathcal{P}(Y)}
W_1\big(\mu\otimes\mathcal{T}^{*}(\nu\mid\mu),\ \mu\otimes F(\nu\mid\mu)\big)<\varepsilon' .
\end{equation}
The block $\mathcal{T}^{*}$ is thus chosen once and for all, before any text is fixed: the supremum
in \eqref{eq:unif} ranges over all pairs of probability measures, in particular over the empirical
measures of all pairs of texts.

Let now $(t_1,t_2)$ be any pair of texts satisfying the separation hypothesis, with token embeddings
$t_1=(x_1,\dots,x_n)$ and $t_2=(y_1,\dots,y_m)$, so that $|x_i-x_{i'}|\ge\delta$ for $i\ne i'$ and
$|y_j-y_{j'}|\ge\delta$ for $j\ne j'$, and let
\[
\mu:=\frac1n\sum_{i=1}^{n}\delta_{x_i},
\qquad
\nu:=\frac1m\sum_{j=1}^{m}\delta_{y_j}
\]
be the associated empirical measures. Put
\[
P:=\mu\otimes F(\nu\mid\mu),
\qquad
P^{*}:=\mu\otimes\mathcal{T}^{*}(\nu\mid\mu),
\qquad
S:=\{x_1,\dots,x_n\}\times\{y_1,\dots,y_m\} .
\]

\emph{The texts have bounded length.}
The balls $B(x_i,\delta/2)$, $1\le i\le n$, are pairwise disjoint, and each is contained in
$B\big(x_1,\operatorname{diam}X+\delta/2\big)$. Comparing Lebesgue measures,
\[
n\,\omega_d\Big(\frac{\delta}{2}\Big)^{d}\le\omega_d\Big(\operatorname{diam}X+\frac{\delta}{2}\Big)^{d},
\qquad\text{hence}\qquad
n\ \le\ \Big(\frac{2\operatorname{diam}X+\delta}{\delta}\Big)^{d}=N .
\]
In particular $\mu(\{x_i\})=1/n\ge1/N$ for every $i$. No restriction on the length of the texts has
been imposed: it is the separation hypothesis that bounds $n$, since $\delta$-separated points in a
compact set are finite in number.

\emph{Both measures live on the same $\delta$-separated grid.}
Both $P$ and $P^{*}$ have first marginal $\mu$, hence give full mass to $\{x_i\}_i\times Y$.
Moreover $F(\nu\mid\mu)(x)$ is concentrated on $\operatorname{supp}\nu=\{y_j\}_j$ for $\mu$-a.e.\ $x$,
by the support condition in the definition of a conditional system, while
$\mathcal{T}^{*}(\nu\mid\mu)(\cdot\mid x)$ is absolutely continuous with respect to $\nu$ by
definition of the softmax normalization, hence also concentrated on $\{y_j\}_j$. Therefore $P$ and
$P^{*}$ are both concentrated on $S$. Furthermore, two distinct points of $S$ differ in at least one
coordinate, so with the metric $d\big((x,y),(x',y')\big)=|x-x'|+|y-y'|$,
\[
d\big((x_i,y_j),(x_{i'},y_{j'})\big)\ \ge\ \min\big\{|x_i-x_{i'}|,\ |y_j-y_{j'}|\big\}\ \ge\ \delta
\qquad\text{whenever }(x_i,y_j)\ne(x_{i'},y_{j'}).
\]

\emph{From $W_1$ to the entries of the joint matrices.}
Let $\gamma$ be any coupling of $P$ and $P^{*}$; since both are concentrated on $S$, so is $\gamma$
on $S\times S$. Write $\Delta:=\{(s,s')\in S\times S:s=s'\}$. For any $A\subset S$ and
$A^{c}:=S\setminus A$, using that the marginals of $\gamma$ are $P$ and $P^{*}$,
\[
P(A)-P^{*}(A)=\gamma(A\times S)-\gamma(S\times A)
=\gamma\big(A\times A^{c}\big)-\gamma\big(A^{c}\times A\big)
\le\gamma\big((S\times S)\setminus\Delta\big),
\]
since both decompositions contain the common term $\gamma(A\times A)$; exchanging $A$ and $A^{c}$
gives the same bound for $P^{*}(A)-P(A)$, so
$\|P-P^{*}\|_{TV}\le\gamma\big((S\times S)\setminus\Delta\big)$. Off the diagonal the two
coordinates are distinct points of $S$, hence at distance at least $\delta$ by the previous
paragraph, and therefore
\[
\int d\,d\gamma\ \ge\ \int_{(S\times S)\setminus\Delta}d\,d\gamma\ \ge\ \delta\,\gamma\big((S\times S)\setminus\Delta\big)\ \ge\ \delta\,\|P-P^{*}\|_{TV} .
\]
Taking the infimum over couplings and using \eqref{eq:unif},
\begin{equation}\label{eq:joint}
\max_{i,j}\Big|P\big(\{(x_i,y_j)\}\big)-P^{*}\big(\{(x_i,y_j)\}\big)\Big|
\ \le\ \|P-P^{*}\|_{TV}
\ \le\ \frac{W_1(P,P^{*})}{\delta}
\ <\ \frac{\varepsilon'}{\delta},
\end{equation}
the first inequality being $\|P-P^{*}\|_{TV}=\sup_A|P(A)-P^{*}(A)|$ applied to singletons.

\emph{From the joint to the conditional matrices.}
Since $P$ and $P^{*}$ share the first marginal $\mu$, and $\mu$ charges each atom $x_i$ with mass
$1/n>0$, their disintegrations along $\mu$ are given on the atoms of $S$ by
\[
F_{ij}(t_2\mid t_1)=\frac{P\big(\{(x_i,y_j)\}\big)}{\mu(\{x_i\})},
\qquad
\mathcal{T}^{*}_{ij}(t_2\mid t_1)=\frac{P^{*}\big(\{(x_i,y_j)\}\big)}{\mu(\{x_i\})} .
\]
Combining this with $\mu(\{x_i\})\ge1/N$ and \eqref{eq:joint},
\[
\max_{i,j}\big|F_{ij}(t_2\mid t_1)-\mathcal{T}^{*}_{ij}(t_2\mid t_1)\big|
=\max_{i,j}\frac{\big|P(\{(x_i,y_j)\})-P^{*}(\{(x_i,y_j)\})\big|}{\mu(\{x_i\})}
\ \le\ N\cdot\frac{\varepsilon'}{\delta}
\ =\ \frac{\varepsilon}{2}\ <\ \varepsilon .
\]
As $(t_1,t_2)$ was an arbitrary pair of texts satisfying the separation hypothesis and
$\mathcal{T}^{*}$ does not depend on it, the proof is complete.
\end{proof}
\end{document}